\RequirePackage{fix-cm}

\documentclass{svjour3}                     
\smartqed  
\usepackage{pgfplots}
\usepackage{float}
\pgfplotsset{compat=newest}
\pgfplotsset{plot coordinates/math parser=false}

\usepackage{graphicx}
\usepackage{amsmath}
\usepackage{mathtools}
\usepackage{xfrac}
\usepackage{pstricks,pst-node,pst-tree,pstricks-add}
\usepackage{booktabs}
\usepackage{algorithm}
\usepackage{algpseudocode}

\def \LA {Learning Automata }
\def \DLA{Distributed Learning Automata }
\def \lri{$L_{R-I} $}
\newcommand{\myvec}[1]{\vec{#1}}
\newcommand{\ra}[1]{\renewcommand{\arraystretch}{#1}}

\begin{document}

\title{Extended Distributed Learning Automata}
\subtitle{A New Method for Solving Stochastic Graph Optimization Problems }


\author{M.R.Mollakhalili Meybodi         \and        M.R.Meybodi 
}

\institute{M.R.Mollakhalili Meybodi \at
              Computer Engineering Department, Science and Research Branch, Islamic Azad University \\
              Tel.: +98-21-64542724\\
              Fax: +98-21-64542724\\
              \email{m.meybodi@srbiau.ac.ir}     
           \emph{Present address:} Islamic Azad University, Meybod Branch  
           \and
           M.R.Meybodi \at
              SoftComputing Lab. Amirkabir University Of Technology\\ \email{mmeybodi@aut.ac.ir} 
}

\date{Received: date / Accepted: date}
\maketitle
\begin{abstract}
In this paper, a new structure of cooperative learning automata so-called extended learning automata (eDLA) is introduced. Based on the proposed structure, a new iterative randomized heuristic algorithm for finding optimal sub-graph in a stochastic edge-weighted graph through sampling is proposed. It has been shown that the proposed algorithm based on new networked-structure can be to solve the optimization problems on stochastic graph through less number of sampling in compare to standard sampling. Stochastic graphs are graphs in which the edges have an unknown distribution probability weights. Proposed algorithm uses an eDLA to find a policy that leads to an induced sub-graph that satisfies some restrictions such as minimum or maximum weight (length). At each stage of the proposed algorithm, eDLA determines which edges to be sampled. This eDLA-based proposed sampling method may result in decreasing unnecessary samples and hence decreasing the time that algorithm requires for finding the optimal sub-graph. It has been shown that proposed method converge to optimal solution, furthermore the probability of this convergence can be made arbitrarily close to 1 by using a sufficiently small learning rate. A new variance-aware threshold value was proposed that can be improving significantly convergence rate of the proposed eDLA-based algorithm.  It has been shown that the proposed algorithm is competitive in terms of the quality of the solution
\end{abstract}

\section{Introduction}
\label{intro}
Automata models for learning systems introduced in the 1960s were popularized by Narendra in\cite{Narendra1974}  as learning automata (LA).  Afterward, there have been many advances in the theory and applications of these learning models\cite{Thathachar2002}. Especially, groups of LA forming teams or feed forward networks have been shown to converge to desired solutions by proper choice of learning rate \cite{Thathachar2002,Meybodi2002}. 
\\Fundamentally, \LA are “simple agents for doing simple things”. The full potential of learning automaton is realized when a group of automata interact to each other to solve the specified problem. One of the interconnected structures of learning automata is \DLA, which in \cite{H.Beigy2006} the complete version of it was introduced and applied\cite{Beigy2003}.\\
DLA as a network of learning automata which collectively cooperate in the random environment to solve a particular problem, repeatedly has been used to solve various problems, especially issues related to random graphs with weighted edges or vertices\cite{H.Beigy2006,Meybodi2002,MollakhaliliMeybodi2004}.\\
 One of the obvious drawbacks of DLA, as a distributed multi-agent system, is that, next active agent is selected directly by the action taken by the current active agent. In other words, in DLA, at any time only one LA is active and able to take action on the environment. Next active LA specified by the action is chosen by the current active LA.   
 Although this network of LA has been able to solve various problems in random graphs such as finding shortest path\cite{MollakhaliliMeybodi2004,Beigy2003,A.Alipour2005,Motevalian2006}, or problems in other areas such as web documents clustering\cite{Saati2005}, web page rank\cite{Anari2007}, link prediction in adaptive web sites\cite{MollakhaliliMeybodi2008}, user modeling in adaptive hypermedia\cite{MollakhaliliMeybodi2012}, web usage mining \cite{BaradaranHashemi2007} and so on, but LA activation mechanism in DLA, has limited its application. \\
 In this paper, the extended approach to activation mechanism of LAs in DLA is introduced, referred to as eDLA. Furthermore we introduce an algorithm based on eDLA to finding optimal sub graph in random graphs. The rest of paper is organized as follows. In section 2 we first briefly deal with learning automata theory, graph theory and stochastic graph. In addition to this, eDLA is introduced in this section. In section 3, our eDLA-based proposed algorithm is presented. Section 4 is aimed at proving the convergence of the proposed algorithm. Section 5 evaluates the performance of the proposed algorithm through the simulation experiments. Finally in the last section we discuss the findings of this paper and present ideas for future researches.
\section{Random Graph, Learning Automata and DLA}
\label{sec:2}
In order to better understand the proposed algorithms and methods discussed in this paper, the concepts of random graph, learning automata and distributed learning automata are briefly described in this section.

\subsection{Stocahstic Graph:}
\label{sec:StochasticGraph}

A stochastic graph G is denoted by a triple $G=(V,E,Q)$ where $E \subseteq  V \times V$ is a set of edges, $ V = \{v_1, v_2, ... , v_n\}$
is a set of nodes,  and $n \times n$ matrix $Q_{n \times n}$ is the probability distribution describing the statistics of edge lengths where $n$ is the number of nodes. For each edge $e_{(i,j)} \in E$ associated weight $w_{ij}$ is assumed to be a positive random variable with $q_{ij}$ as its probability density function (PDF), which is assumed to be unknown in this paper. As a result, any propsed algorithm in this paper is
based on the assumption that the $q_{ij}$ isnot known apriori.
\subsection {Stochastic Learning Automata:}
\label{LA}
 
A stochastic learning automaton is an adaptive decision making unit that improves its performance by learning how to choose the optimal action from a finite set of allowable actions through repeated interactions with a random environment. 
\\At each instant, automata choose an action from its available actions, based on a probability distribution kept over the action set. Selected action is served as the input to the random environment. The environment responds to the taken action with a reinforcement signal called as stochastic response. Based on the reinforcement feedback from the environment, the action probability vector is updated based on learning process which is known as reinforcement scheme. The objective of this process is to find the optimal action from the actions-set so that the average penalty received from the environment is minimized. 
\\From a mathematical point of view, a \LA can be considered as a finite state machine that can be described by a 5-tuple as follow:
\begin{equation}
\label{equations.SLA}
S\equiv \{\myvec{ \alpha} ,\myvec{\beta},F,G,\myvec{\varphi}\} 
\end{equation}
where:

\begin{itemize}
  \item $\myvec{\alpha} \equiv \{\alpha_1,\alpha_2,\ldots,\alpha_r\}$ is a set of actions (or outputs of automaton). The output or action of an automaton at the instant $n$ is an element of finite set \myvec{\alpha} which is denoted by $\alpha(n)$ .
  \item $r$ is the number of available actions
 \item $\myvec{\beta} \equiv \{\beta_1,\beta_2,\ldots,\beta_m\}$ is a set of responses from environment (or inputs of automaton). The input $\beta(n)$ from the environment is an element of set $\myvec{\beta}$ which could be either a finite set or infinite set such as interval on the real line.
\item $F \equiv \varphi \times \beta \longrightarrow \varphi$ is a function that maps the current state and input into the next state
\item $G \equiv \varphi  \longrightarrow \alpha$ is a function that maps the current state to next output
\item  $\myvec{\varphi(n)} \equiv \{\varphi_1,\varphi_2,\ldots,\varphi_k\}$ is a set of internal states of automaton at any instant $n$
\end{itemize}
 $\myvec{\alpha}$ is a set of outputs (actions) of automaton. At each step, automaton chooses one of these actions ($r$ is the number of possible actions) and applies it to the environment. $\myvec{\beta}$ identifies set of inputs to the automaton. $F$ and $G$ are functions that map the current input to the next output (next action). If $F$ and $G$ are deterministic, then automaton is a deterministic automaton. In a deterministic automaton if current state and input of automaton are given, then next state and output of automaton are to be determined uniquely. If the mappings $F$ and $G$  are nondeterministic, automaton is a stochastic automaton (or SLA). Stochastic Learning Automata can be classified into two main families: Fixed Structure (FSLA) and Variable Structure (VSLA). \\
\subsubsection{Environment:}
The mathematical model of an environment can be described by a 3-tuple $ \{\myvec{ \alpha} ,\myvec{\beta},\myvec{c}\} $ where 
\begin{itemize}
\item $\myvec{\alpha} \equiv \{\alpha_1,\alpha_2,\ldots,\alpha_r\}$ is a final set of inputs
\item $\myvec{\beta} \equiv \{\beta_1,\beta_2,\ldots,\beta_m\}$ is a set of values that can be taken by the reinforcement signal
\item $\myvec{c} \equiv \{c_1,c_2,\ldots,c_r\}$ is a set of the penalty probabilities where the element $c_i$ is associated with the given action $\alpha_i$
\end{itemize}
Based on the penalty probabilities, random environments can be classified into stationary and non-stationary ones. In a stationary random environment the penalty probabilities are constant, whereas in a non-stationary random environment these probabilities are varying with time.
Based on the number of reinforcement signal, $\myvec{\beta}$, the environments can be classified into P-model-model, Q-model and S-model. The environments, in which the reinforcement signal can only take two binary values 0 and 1, are referred to as P-model environments. Another class of the environment allowing a finite number of the values in the interval [0,1] can be taken by the reinforcement signal. Such an environment is referred to as Q-model. In S-model environments, the reinforcement signal lies in interval [a,b]. \\
In a P-model environment penalty probability associated with an action can be represented as follow:
\begin{equation}
\label{LA.penaltyEquation}
c_i \equiv Prob\{\beta(n)=1|\alpha(n)=\alpha_i\} , i \in\{1,2,\ldots,r\}
\end{equation}
The relationship between the learning automaton and its random environment has been shown in figure \ref{figures.LAandEnvironment} 
\begin{figure}
\centering
\includegraphics[scale=0.5]{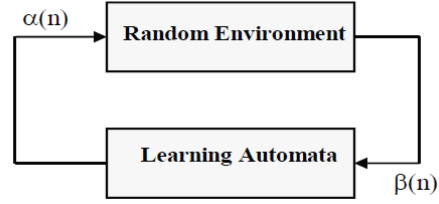} 
\caption{The relationship between the learning automaton and its random environment}
\label{figures.LAandEnvironment}
\end{figure}

Variable structure learning automata are represented by a triple$(\myvec{\beta},\myvec{\alpha},T)$ where $\myvec{\beta}$ is the set of inputs, $\myvec{\alpha}$ is the set of actions, and $T$ is learning algorithm. The learning algorithm is a recurrence relation which is used to modify the action probability vector. 
\subsubsection{Linear Learning Algorithm: }
\label{LinearLearningAlgorithm}
Let  $\alpha_{i}(k) \in \myvec{\alpha}$ and $\myvec{p}(k)$ denote the action chosen at instant $k$ and the action probability vector on which the chosen action is based, respectively. The recurrence equation shown by (\ref{equation.LA.reward}) and (\ref{equation.LA.penalt}) is a linear learning algorithm by which the action probability vector  $\myvec{p(k)}$ is updated. Let be the action chosen by the automaton at instant k. 

\begin{equation}
\label{equation.LA.reward}
p_j(k+1)=\begin{cases}
       (1-a)\times p_j(k)+a       &j=i\\
      (1-a)\times p_j(k)     & \forall j \neq i
      \end{cases}
\end{equation}
When the taken action $\alpha_{i}(k)$   is rewarded by the environment (i.e.,$\beta(k)=0$ ) and 
\begin{equation}
\label{equation.LA.penalt}
p_j(k+1)=\begin{cases}
       (1-b)\times p_j(k)       &j=i\\
      (1-b)\times p_j(k) +\frac{b}{r-1}   & \forall j \neq i
      \end{cases}
\end{equation}
When the taken action $\alpha_{i}(k)$   is penalized by the environment (i.e.,$\beta(k)=1$ )\\
r is the number of actions chosen by the automaton. If $a=b$, the  recurrence equations (\ref{equation.LA.reward}) and (\ref{equation.LA.penalt})are called linear reward-penalty ($L_{R-P}$) algorithm. If $a\gg b$  then the given equations are called linear reward-$\epsilon$penalty ( $L_{R-\epsilon P}$), and finally if $b = 0$ then they are called linear reward-Inaction ($L_{R-I}$ ). In the latter case, the action probability vectors remain unchanged when the taken action is penalized by the environment. 
\subsubsection{Variable action set Learning Automaton: }
\label{variableactionset}
If the number of actions of learning automaton is varying in time, it is called the variable action set learning automaton. The absolute expediency and $\epsilon$-optimality of this type of learning automata under the $L_{R-I}$  reinforcement scheme is shown in \cite{ThathacharMALandHarita1987}.
Assume that $\myvec{\alpha} \equiv \{\alpha_1,\alpha_2,\ldots,\alpha_r\}$ denotes the set of actions of a learning automaton and $V(n) \subset \myvec{\alpha} $ is a non-empty subset of the actions at time . $V(n)$ represents the available (or selectable) actions of the learning automaton at any time $n$  that is called active actions. Selecting the elements of $V(n)$  is done randomly by an external factor. The procedure of selecting an action and updating the action probability vector in this type of learning automaton can be described as follows: let at time$n$ , $V(n)$  be the set of active actions and $K(n)=\sum_{\alpha_i \in V(n)} p_i(n)$ represents the sum of probability of active actions. Before the select of an action, the active actions probability vector is \textit{scaled} by the equation (\ref{equation.VASScale})
\begin{equation}
\label{equation.VASScale}
\forall \alpha_i \in V(n)  :  \hat{p}_i(n)=\frac{p_i(n)}{K(n)}
\end{equation}
Afterward, the learning automaton randomly selects an action according to the scaled action probability vector $\hat{\myvec{p}}(n) \equiv \{\hat{p}_i(n)|\alpha_i \in V(n) \} $. According to the received response from the environment, scaled action probability vector $\hat{\myvec{p}}(n)$ is updated. Finally the active action probability vector  $\hat{\myvec{p}}(n)$ is \textbf{\textit{rescaled}} as mentioned in the equation (\ref{equation.VASrescale})
\begin{equation}
\label{equation.VASrescale}
\forall \alpha_i \in V(n)  :  \hat{p}_i(n)=p_i(n) \times K(n)
\end{equation}

\subsubsection{Distributed Learning Automata: }
Distributed Learning Automata (DLA) is a network of automata which collectively cooperate to solve particular problem. A DLA can be modeled by a directed graph in which the set of nodes of graphs constitutes the set of automata and the set of outgoing edges for each node constitutes the set of actions for corresponding automaton. When an automaton selects one of its actions, another automaton on the other end of the edge corresponding to the selected action will be activated. An example of DLA is given in figure \ref{figures.DLA}. In this example, if automaton $LA_1$ selects action $\alpha {13}$ , then automaton $LA_3$ will be activated. Activated automaton $LA_3$ chooses one of its actions and so on. At any time only one automaton in the network will be activated. 

\begin{figure}
\centering
\includegraphics[scale=0.6]{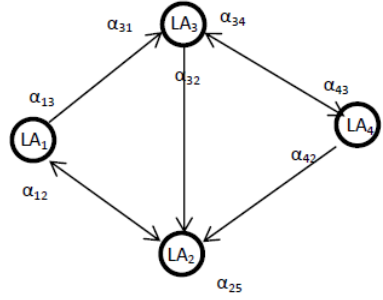} 
\caption{Distributed Learning Automata}
\label{figures.DLA}
\end{figure}

Formally, a DLA can be embedded in a graph and can be defined by a 4-tuple  $(\myvec{A},E,T,A^0)$ where $\myvec{A}=\{A_1,A_2,\ldots,A_n\}$is the set of learning automata, $E \subseteq A \times A$ is the set of the edges in which edge $e_{(i,j)}$ corresponds to the action $a_{ij}$ of the automaton  $A_i$.T is the set of learning schemes with which the learning automata update their action probability vectors, and $A_0$ is the root automaton of DLA from which the automaton activation is started. The operation of DLA can be described as follows: 
\begin{itemize}
\item   On each activation of root node, one of its outgoing edges (one action of root automaton) is chosen on the basis of action probability vector of root node automaton. 
\item The selected edge activates one learning automaton on the other end of the selected edge 
\item This automaton also selects an action which results in activation of another automaton 
\item This process is repeated until a leaf node is reached
\item The leaf node, is a node that interacts to environment 
\end{itemize}
Restricted version of DLA is introduced at first in \cite{Sato1999} where underlying graph is a DAG and $L_{R-I} $ algorithm with decaying reward parameters is used as learning algorithm. Meybodi-Beigy introduce a DLA in which the underlying graph is not restricted to a dag, but in order to restrict a learning automaton to appear more than once in any path, the LA with changing number of actions are used \cite{H.Beigy2006}
\section{Extended Distributed Learning Automata: eDLA }
\label{section.eDLA}
eDLA, is a network of interconnected cooperative learning automata, supervised by a set of communication rules governing the order of operation of Learning Automaton. In eDLA, each learning automaton has an activity level, changing according to the problem solved by eDLA and communication rules. At any time, only one automaton has high activity level, which can perform an action on the environment.\\
Formally, an eDLA similar to DLA, can be embedded in a graph and defined by a 7-tuple $eDLA \equiv \{A,E,S,P,S^0, F,C \}$
where A is a set of vertices of graph and E is the set of edges of $G \equiv (A,E)$ so called as \textit{communication graph}.
$S \equiv \{s_1,s_2,\ldots,s_n\}$ is a set of activity level corresponding to each of learning automaton in eDLA and activity level of learning automaton  $A_i$ , is defined by $s_i$. Each automaton can be in one of the following activity levels: \textbf{\textit{Passive}}, \textbf{\textbf{Active}}, \textbf{\textit{Fire}} and \textit{\textbf{Off}} represented by \textbf{Pa}, \textbf{Ac}, \textbf{Fi} and \textbf{Of} respectively.  
\\
\textit{Fi} is the highest level of activity of \LA and \textit{Pa} and \textit{Of} are the lowest level of it. The difference between \textit{Pa} and \textit{Of} is that, when a \LA downgrade to \textit{Of} level of activity, at this level, cannot change to other levels, but in \textit{Pa} level, upgrades to higher level is possible. At any time only one of the automata in the eDLA will be in \textit{Fi} level of activity that called ‘\textit{Fire Automaton}’. In an eDLA,  is the finite set of rules that governs the activity levels of each automaton. These rules based on the current activity level of each automaton and activity levels of its parents or adjacent automata in G, determines the next level of activity. These rules may vary depending on the problem being solved by eDLA. 
\\
$S^0 = (s_1^0,s_2^0,\ldots,s_n^0)$ is called the initial state of the eDLA. $F \equiv \{ S^F |S^F = (s_1^F,s_2^F,\ldots,s_n^F) \}$ is called “final conditions”. F is a set of circumstances in terms of activity levels of automata, that if realized at least one of them, eDLA is transferred to the final state. Obviously  has at least one item, that is a situation where activity of all automata in eDLA is \textit{Of}.
C is a special function that selects one automaton in eDLA. This selection is done based on the current activity level of each automaton in eDLA and the associated problem that eDLA is designed to solve it. This function is called Fire Function or briefly Fire. In the following discussion we will explain more this function. 
The number of actions for a particular automata, is equal to degree (in situation that G is a undirected graph) or out degree (if G is a directed graph) of corresponding node in communication graph G. 
\\The eDLA works as follows: 
\begin{itemize}
\item   At first, a network of learning automaton which is isomorphic to an input graph is created. In this network, each node is a learning automaton and each edge (or outgoing edge) of this node is one of the actions of this learning automaton. 
\item ·The eDLA starts from state $S^0$ and based on the rules $P$ , the activity level of each learning automaton in eDLA changes. These changes make the state of eDLA change from $S^0$ to $S^1$ . The transition from state  $S^i$ to  $S^{i+1}$ continues until the eDLA reaches to a final state.
\item At each time, a node that has a high activity level (\textit{Fi}), selects an action from its available actions set and applies it to the environment. After this, fire automaton switches to low activity level \textit{Of}. Moreover, based on the communication rules, a set of automata that adjacent to \textit{Fire Automaton} and has a lowest level of activity (\textit{Pa}), upgrades to next high level of activity (\textit{Ac}). The next  fire automaton is selected from the set of automata with \textit{active} level (\textit{Ac}) of activity. This selection often is done randomly. 
\end{itemize}
\paragraph{Run}: A run is the change of activity level of all automata in an eDLA to \textit{Of} level until at least one of the final conditions is realized. The activity-level change is done based on the activity-change-cycle $Passive \rightarrow Active \rightarrow Fire \rightarrow Off $ . At the beginning of each run, there is at least one automaton with Ac level of activity. This automaton is called the root and is represented by . All of these automata are denoted by $S^0$ . 
\paragraph{Instantaneous Description:} is an ordered 4-tuple $D^t=(D_{Of}^t,D_{Fi}^t,D_{Ac}^t,D_{Pa}^t)$ where $D_i^t \subseteq A (\forall i \in  \{Pa,Ac,Fi,Of\})$ and $D_i^t \bigcap D_j^t =\phi (\forall i ,j \in  \{Pa,Ac,Fi,Of\})$
\\When a learning automaton in the eDLA performs one of its actions, instantaneous description will change based on the\textit{ P} rules set. Moreover, the instantaneous description changes when the activity level of a learning automaton in eDLA changes to \textit{Fi} level. Initial instantaneous description, $D^0$, of eDLA is defined as $D^0=(D_{Of}^0,D_{Fi}^0,D_{Ac}^0,D_{Pa}^0)=(\phi,\phi,\{A_0\},A-\{A_0\})$ and default final instantaneous description is defined as $D^{final}=(D_{Of}^{final},D_{Fi}^{final},D_{Ac}^{final},D_{Pa}^{final})=(A,\phi,\phi,\phi)$
Based on the above definitions and notations, a run in eDLA can be described by a sequence of instantaneous descriptions, such as follows: 
$Run \equiv D^0 \succ^{fire}  D^1 \succ^{action} D^2 \succ^{fire} \ldots \succ D^n \in F  $

\begin{figure}
\centering
\includegraphics[scale=0.7]{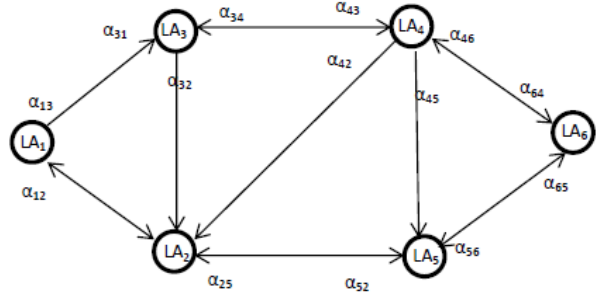} 
\caption{Communication graph of an eDLA}
\label{figures.exampleofeDLA}
\end{figure}
For a better understanding of definitions, see figure \ref{figures.exampleofeDLA} for an example of a run 

$
(\phi,\phi,\{1\},\{2,3,4,5,6\}) \succ^{fire(1)} (\phi,\{1\},\{2,3\},\{4,5,6\}) \succ^{action(a_{12})}\\
(\{1\},\phi,\{2,3\},\{4,5,6\}) \succ^{fire(3)} (\{1\},\{3\},\{2,4\},\{5,6\})\succ^{action(a_{32})} \\
(\{1,3\},\phi,\{2,4\},\{5,6\}) \succ^{fire(4)} (\{1,3\},\{4\},\{2,5,6\},\phi)\succ^{action(a_{46})} \\
(\{1,3,4\},\phi,\{2,5,6\},\phi) \succ^{fire(2)} (\{1,3,4\},\{2\},\{5,6\},\phi)\succ^{action(a_{25})} \\
(\{1,3,4,2\},\phi,\{5,6\},\phi) \succ^{fire(6)} (\{1,3,4,2\},\{6\},\{5\},\phi)\succ^{action(a_{65})} \\
(\{1,3,4,2,6\},\phi,\{5\},\phi) \succ^{fire(5)} (\{1,3,4,2,6\},\{5\},\phi,\phi)\succ^{action(a_{32})} \\
 \succ^{no action} (\{1,3,4,2,6,5\},\phi,\phi,\phi)\\
$

Thus, the eDLA is a dichotomous structure consisting of a fire function and a network of learning automata as a team that cooperate together to solve a specific problem. The fire function is differing depending on the problem, but in general the task of fire function is to determine the automaton which should do an action in a random environment. Figure(\ref{figures.eDLAStructure}) shows the general structure of an eDLA. 

\begin{figure}
\includegraphics[scale=0.8]{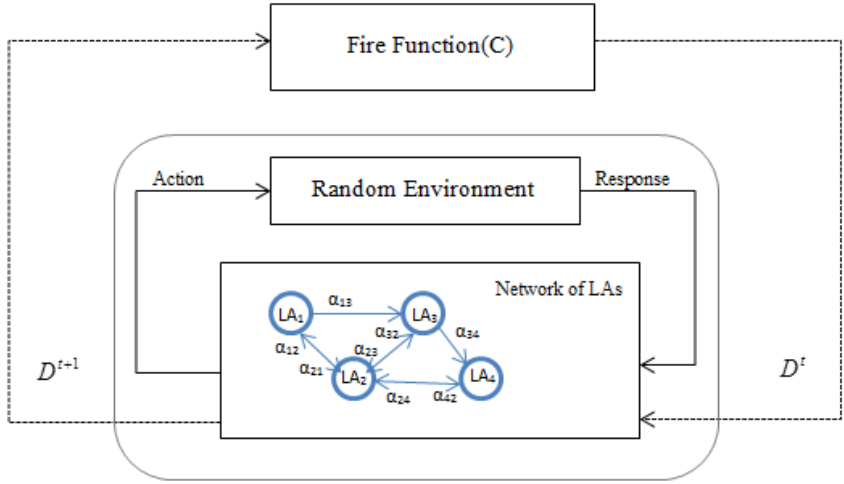} 
\caption{General structure of an eDLA }
\label{figures.eDLAStructure}
\end{figure}
 
 \subsection{Solving stochastic induced sub-graph problem by eDLA}
 \label{subsection.proposedAlgorithm}
As mentioned earlier, stochastic graph is a suitable model for describing many real-world problems. For example, in computer networks, communication links between switching elements such as routers and switches can be modeled as a random edge-weighted graph. Suppose we want to find an induced sub-graph with certain restrictions in such stochastic graph. For example we might to find a sub-graph that is spanning tree with minimum total weight. Finding induced sub-graph that satisfies certain restrictions in a stochastic graph with unknown distribution random edge-weighted, in general is a \#P-Hard problem and we called this as solving stochastic induced sub-graph problem in the rest of this paper. One solution to this problem is sampling from edges in order to estimate deterministic mean value of edge weight and then to solve problem in deterministic equivalent graph obtained by replacing stochastic edge by resulting deterministic edge. This requires a large number of samples. Another solution to this problem is purposeful sampling. One of these methods is the sampling by eDLA. In this way, by using an eDLA, sampling is done from more efficient edges. This sampling method may result in running time of algorithm by decreasing unnecessary samples. The general procedure used by eDLA-based algorithm can be shown in Algorithm(\ref{figures.eDLAAlg})

\begin{algorithm} 

\begin{algorithmic}[1]
\caption{:e\textsc{DLA-Based Algorithm Pseudo Code}}
\label{figures.eDLAAlg}
        \Require Stochastic Graph $\mathcal{G}= (V,E)$, $P_s$; $K_s$; 
        
        \State Assign a LA to each node of $\mathcal{G}$ and make an eDLA=(V,E) 
        \State Let S denotes the constructed sub-graph;
        \State Let $W_s$ denotes the weight of constructed sub-graph;
        \State $P \gets 0$
        \State $K \gets 0$
        \While{$K \leq K_s \vee P \leq P_s$} 
            \State $S \gets \phi;$
            \State $W_s \gets 0;$
            \State $v$=ROOT() \Comment ROOT() determines a  start node according to the problem 
            \State $A_P \gets V-\{v\}$ \Comment  $A_P$ = set of LA with ``Passive'' level
            \State $A_A \gets {v}$  \Comment $A_P$ = set of LA with ``Active'' level
            \State $A_O \gets \phi$  \Comment $A_O$ = set of LA with  ``Off'' level
            \State $A_F \gets Fire(eDLA);$ \Comment $A_F$ is a LA that has ``Fire'' level
            \While{$A_F \neq \phi$}
                \State $A_A=Update(A_F,G)$ \Comment Activate according to adjacency with  $A_F$
                \State $\alpha \gets SelectAction(A_F);$
                \State disable action $\alpha$ in adjacent active automata \Comment for prevent loop cration
                \State $S \gets S \cup edge(A_i,\alpha);$
                \State $W_s \gets W_s + SampledWeight(edge(A_F,\alpha));$
                \State $A_O \gets A_O \cup A_F;$           
                \State $A_F \gets Fire(eDLA);$    
            \EndWhile
            \State $K \gets K+1$
            \State response=Evaluate($W_s$,Threshold,K); \Comment response=reward or penalty
            \ForAll{$LA \in A_O$}          
            \If{$resonse= reward$}
                 \State reward selected action by LA \Comment by \lri scheme
            \EndIf
                \State enable all disabled actions \Comment for next run            
            \EndFor
            \State Threshold=Update($W_s$,Threshold,K);
            \State $P=\prod{P(A_m,\alpha_n)} \;\;\; \forall \;{edge(m,n) \in W_s}$      
         \EndWhile
\end{algorithmic}
\end{algorithm}

 As can be seen in the pseudo-code in Algorithm(\ref{figures.eDLAAlg}), several functions have been used in the algorithm: 
 \begin{itemize}
 \item \textit{Fire function}: this function, at each stage, indicates one automaton as fire automaton according to the activity level of each automaton and communication rules used by eDLA (e.g. it selects randomly an automaton from automata with \textit{‘Ac’} level of activity). Furthermore, based on the P rules of eDLA and fire automaton, other adjacent automata with fire automaton, which have \textit{‘Pa’} level of activity upgrade to \textit{‘Ac’} level 
 
 \item \textit{Selectaction function}: this function takes the number of current fire automaton, as input parameters, and selects one action from available actions set of this automaton. This selection is done according to the selection action probability vector of specified automaton. After this, activity level of specified fire automaton downgrade to \textit{‘Of’}. After selecting an action, and if the induced graph must be acyclic, the actions that can cause loop must be disabled in the set of actions in other automata with \textit{‘Ac’} level of activity. Algorithm which used to avoid the creation of loop will be described later. 
 \item \textit{Evaluate function}: This function is used by environment to evaluate selected action of eDLA. The action of eDLA is a sequence of actions that represents a particular induced sub-graph in the stochastic graph. The environment uses the sample length of this induced-graph ($W_S$) to produce its response. This response, depending on whether it is unfavorable or favorable, causes the actions corresponding to induced sub-graph edges, penalized or rewarded respectively. After evaluation of eDLA action, all of the disabled actions are enabled and activity level of all automata in eDLA is set by \textit{‘Pa’}. 
 \item \textit{Update function}: let us assume the induced sub-graph $s_i$  is selected at stage $k$ the average weight of all previously constructed sub-graphs by eDLA is called dynamic threshold and denoted by $T$.  At each stage $n > 1$ dynamic threshold $T$ is computed as 
 \begin{equation}
T=T_n= \frac{1}{k} \displaystyle\sum_{j=1}^{n}W_{s_i}
\label{equation.thresholdDefine}
 \end{equation}
 Where in (\ref{equation.thresholdDefine}), $W_{s_i}$ denotes the weight of sub-graph $s_i$ and defined as $W_{s_i}=\sum_{\forall e_{(s,t)} \in s_i}{w(s,t)}$. $w(s,t)$ is a positive random variable and show the wight of $e_{(s,t)}$ at $k^{th}$ iteration of algorithm.
 \end{itemize}
The process of doing \textit{Fire} function and \textit{Selectaction}  function is repeated until all automata in eDLA downgrade to \textit{‘Of’} activity level. After this, length of the induced sub-graph is computed and compared with the dynamic threshold. Depending on the problem associated to eDLA and constraint that should be satisfied by sub-graph and the result of this comparison, all the learning automata with \textit{‘Of’} activity level update their action selection probability vector (\textit{Evaluate} function) and dynamic threshold is updated according to \textit{Update} function described above
The process of constructing sub-graph is repeated until the stopping condition is reached where at this point the last constructed sub-graph is the sub-graph that satisfies specific constraint among all induced sub-graphs. Stop condition can be either the probability of choosing the edges of the sub-graph called the sub-graph probability greater than a certain pre-defined threshold or pre-specified numbers of sub-graphs selected. 
\subsection{Improvement:}
our Studies show that as the dynamic threshold value is close as to the weight of optimal sub-graph, eDLA learning ability is reduced. To solve this problem, rather than the dynamic threshold which is an average of obtained sub-graph weights, we use  a different criterion. This criterion should be able to consider mean and variance of weight of sampled sub-graphs. However in order to reduce the computational complexity of proposed improvement and applicability of these:
\begin{itemize}
\item First, we need to use a simple estimate such as stochastic gradient for mean and variance 
\item	Secondly, the weight of individual edge is not observable and the total weight, given by the sum of the weights of all edges on the sub-graph is subsequently observed

\end{itemize}
Thus, the stochastic graph is considered as a black box and at each step, any proposed improvement, only has access to weight of current sampled sub-graph, approximated mean and variance of the previously sampled sub-graphs. In this way, no additional assumption imposed on the stochastic graph and learning automata simplicity is not violated.
\\
To obtain a new criterion, at each stage $k$  we have to do this: 
\begin{itemize}
\item The $Err_k$  is calculated as differences between the weight of obtained sub-graph ($W_k$) at stage $k$   and estimated mean of previously obtained sub-graph weights ($T_{k-1}$). In other words $Err_k=W_k-T_{k-1}$ 
\item The estimated value of mean is updated by  $T_k=T_{k-1}+\alpha \times Err_k$    
\item The estimated value of variance is update by equation  $$Var_k=Var_{k-1}+\beta\times(|Err_k|-Var_{k-1})$$ It is clear that $Var_k$  estimates the mean deviation instead of standard deviation and we have $mdev^2=(\sum{|W-T|})^2 \geq \sum{{|W-T|}^2}=\delta^2=sdv^2$. This shows that mean deviation is greater than standard deviation and hence it is more conservative than the standard deviation
\end{itemize}
In order to evaluate eDLA action in proposed algorithm instead of dynamic threshold we use a linear combination of estimated mean and estimated variance. In minimization problem we use $\sfrac{T_k}{2}+2\times Var_k$ as upper bound for comparison. In maximization problem we use $\sfrac{T_k}{2}-2\times Var_k$ as lower bound for comparison. 

\section{Asymptotic Behavior of eDLA-based Proposed Algorithm: }
\label{section.mathematics}
Learning automata can automatically improve their behavior based on a response from a random stationary environment, but when connected with each other, their behavior becomes complex and hard to analyze \cite{Sato1999}. In this section we analyze the behavior of eDLA-based proposed algorithm (Algorithm(\ref{figures.eDLAAlg})) for finding optimal sub-graph in stochastic graph. For this purpose, we use the similar method proposed by \cite{Narendra1974} and \cite{Lakshmivarahan1976} to study the convergence of algorithm and in \cite{H.Beigy2006} used to study the asymptotic behavior of DLA. Basically two types of convergence for different reinforcement schemas are reported in the literatures\cite{Narendra1974}. In first mode of convergence that occurs typically in the case of expedient schemas (e.g. $L_{R-P}$), the distribution functions of the sequence of action probabilities converge to a distribution function. For example, it can be shown that when the $L_{R-P}$ algorithm is used, the selection action probability $p(n)$ converges to a random variable with a continuous distribution and the variance will be computable and hence can be made as small as desired by the proper choice of parameters of  $L_{R-P}$ scheme. A second and stronger mode of convergence occurs in the case of $\epsilon$-optimal schemes (such as  $L_{R-I}$ scheme). In this mode by using the martingale convergence theorem it can be proved that the sequence of action probabilities converges to a limiting random variable with probability one \cite{Lakshmivarahan1976}
 \\ In \cite{H.Beigy2006}, authors introduce path probability concept in DLA. Then by a method similar to \cite{Norman1968} they prove that if the graph has unique shortest path, the path probability of the shortest path converges to one in the proposed algorithm and hence the proposed algorithm converges to shortest path with probability as close as to unity. In this section we will use the proposed method in \cite{H.Beigy2006} to prove convergence of our proposed algorithm. 
 \\Throughout this section we consider following notations: assume that in the proposed algorithm (Algorithm(\ref{figures.eDLAAlg})) up to stage $k$,  sub-graph $s_i$ (for i=1,2, \ldots ,r which r is the number of sub-graphs) is selected $k_i$ times, where $k=\sum_{i=1}^{r}k_i$. Also assume that sub-graph $s_i$ is selected at stage $k+1$. Let $W_{s_i}(j)$ denotes the weight of $s_i$ (for i=1,2, \ldots ,r)  at the $j ^{th}$ sampling time (for j=1,2, \ldots ,k) according to proposed algorithm, we have:
 \begin{equation}
 \label{equation.rewardpenalt.definition}
 \begin{cases}
  d_i(k)=Prob[W_{s_i}(k+1)\; Satisfies\; T_k] \\
 c_i(k)=Prob[W_{s_i}(k+1)\; notSatisfies\; T_k] =1-d_i(k)
 \end{cases}
 \end{equation}
 
 Where $c_i(k)$ and $d_i(k)$ denote penalty and reward probabilities of induced sub-graph $s_i$  at stage $k$ , respectively and $T_k$ is the dynamic threshold value as described in Equation(\ref{equation.thresholdDefine}) (section \ref{subsection.proposedAlgorithm})
 The probability of induced sub-graph $s_i$ defined as the product of probability of choosing the edges of $s_i$  and denoted by $q_i$. $q_i(k)$ represents the probability of choosing sub-graph $s_i$ at the stage $k$
 \begin{theorem}
 \label{theory.mainTheorem}
 Suppose that $s_i$ is the optimal induced sub-graph. If  $\myvec{q} = [q_i]_{i=1,2,\ldots,r}$ is updated according to the proposed algorithm then  $\lim_{n \rightarrow \infty}  q_i(n)\stackrel{a.s.}{=}1$
 \end{theorem}
 In other words, theorem \ref{theory.mainTheorem} asserts that the proposed algorithm converges to the optimal sub-graph with a probability as close to 1 as desired. 
 \begin{proof}:
 the proof of this theorem is done in multiple stages: at first, $q_i(k)$is computed as a function of probabilities of actions of learning automata that construct the sub-graph $s_i$.  In second stage, it is shown that for large enough $k$, the probability of choosing the optimal sub-graph by the proposed algorithm is a sub-martingale process. In the next stage of the proof, using the martingale convergence theorem convergence of the proposed algorithm to optimal sub-graph is shown. The proposed algorithm uses $L_{R-I}$ scheme as learning algorithm of automaton. In this scheme the action probability vector of eDLA, $\myvec{q}(n)$ converges to the set of absorbing states with probability one (this was shown at the previous stage of proof) that one of them is the desired one. We can only say that $\myvec{q}(n)$  converges to the desired state with a positive probability. In the final stage of proof, quantify this probability by the method proposed by \cite{Norman1968}. 
 
 \begin{lemma} 
 If $\myvec{q}$  is updated according to the proposed algorithm and $c_i(k)$ and $d_i(k)$ denote penalty and reward probabilities of induced sub-graph $s_i$ at stage k respectively that define by Equation(\ref{equation.rewardpenalt.definition}), then we have
 \begin{equation}
 \label{equation.omidriazisharti}
 E \left[ q_i(k+1)|\myvec{q}(k)\right] =\sum_{j=1}^{r} {q_j(k)\left[c_j(k)q_i(k)+d_j(k) \prod_{e_{(m,n)}\in s_i}\delta_n^m(k)\right]} 
\end{equation}  
Where in Equation(\ref{equation.omidriazisharti}) , $E\left[. \right]$ denotes the mathematical expectation and we have:
\begin{equation}
\delta_n^m(k)=
\begin{cases}
p_n^m(k+1)=p_n^m(k)\times (1-a)+a & e_{(m,n)}\in s_j\\
p_n^m(k+1)=p_n^m(k)\times (1-a) & e_{(m,n)}\notin s_j
\end{cases}
\end{equation}
 \end{lemma}
 \begin{proof}:
 Proposed algorithm uses the  reinforcement \lri scheme to evaluate eDLA action. Assume that at stage $k$ sub-graph $s_j$  is selected by eDLA. If the selected sub-graph is penalized by the random environment (with probability $c_j(k)$), the probability of choosing all sub-graphs in the graph remain unchanged. If the selected sub-graph is rewarded by the random environment (with probability $d_j(k)$), two cases are conceivable: 
 \begin{itemize}
 \item[I]: $s_j$ and $s_i$ have no common edges: in this case all edges forming $s_i$ are penalized and the probability of choosing  $s_i$ decreases.
 \item[II] :$s_j$ and $s_i$ have some common edges:in this case common edges are rewarded and other uncommon edges of $s_i$  are penalized.
 \end{itemize}
 By the above description we have: 
  \begin{equation}
 \label{equation.omidriazisharti:1}
 E \left[ q_i(k+1)|\myvec{q}(k)\right] =\sum_{j=1}^{r} {E \left[ q_i(k+1)|\myvec{q}(k),\alpha_j\right]}\times p(\alpha_j|\myvec{q}(k)) 
\end{equation}  

  In Equation(\ref{equation.omidriazisharti:1}) $\alpha_j$ denotes the all learning automata actions that resulted in the selection of sub-graph $s_j$  by eDLA \\
  But
  \begin{equation}
  \label{equation.omidriazisharti:2}
  p(\alpha_j |\myvec{q}(k))=q_j(k)
  \end{equation}
  And:
  \begin{eqnarray}
  \label{equation.omidriazisharti:3}\nonumber   
 &&E\left[ q_i(k+1)|\myvec{q}(k),\alpha_j\right]= \\
 && c_j(k) q_i(k)+d_j(k)  \prod_{\substack{e_{(m,n)} \in s_i \\e_{(m,n)}\in s_j}}\left(p_n^m(k)\uparrow\right)  \prod_{\substack{e_{(m,n)} \in s_i\\  e_{(m,n)}\notin s_j}}\left(p_n^m(k)\downarrow\right) 
  \end{eqnarray}

where $p_j^i(k)$ represents the probability of selection of edge $e_{(i,j)}$ or equally probability of choosing action $\alpha_j$ by fired automaton $A_i$ at stage $k$. In addition $p_j^i(k)\downarrow$ and  $p_j^i(k)\uparrow$ denote the decreasing and increasing of $p_j^i(k)$ respectively by  \lri reinforcement scheme as follows:
\begin{equation}
p_j^i(k+1)=
\begin{cases}
p_j^i(k)\downarrow=p_j^i(k) \times (1-a) \\
p_j^i(k)\uparrow=p_j^i(k) \times (1-a)+a
\end{cases}
\end{equation}
 By replacing Equations(\ref{equation.omidriazisharti:2}) and (\ref{equation.omidriazisharti:3})  into Equation(\ref{equation.omidriazisharti:1}) we obtain:
  \begin{eqnarray}
 \label{equation.omidriazisharti:endproof} \nonumber
 &&E \left[ q_i(k+1)|\myvec{q}(k)\right] =\\
 && \sum_{j=1}^{r} {q_j(k)\left[c_j(k)q_i(k)+d_j(k) \prod_{\substack{e_{(m,n)} \in s_i \\e_{(m,n)}\in s_j}}(p_n^m(k)\uparrow)  \prod_{\substack{e_{(m,n)} \in s_i \\e_{(m,n)}\notin s_j}}(p_n^m(k)\downarrow) \right]} 
\end{eqnarray}  
We thus proved the lemma \begin{qed} \end{qed}
  \end{proof} 

\begin{lemma}
\label{lemma.nonnegativeboodan}
For a given stochastic graph, when $s_i$ is the optimal sub-graph and $\myvec{q}(k)$ is updated according to the proposed algorithm, the increament in the conditional expectation of $q_i(k)$ subject to $\myvec{q}(k)$,$\bigtriangleup q_i(k)=E \left[ q_i(k+1)-q_i(k)|\myvec{q}(k)\right] $ is always non-negative.
\end{lemma}
\begin{proof}:
based on the proposed algorithm we can see that it atmost chooses $(n-1)$ edges of the stochastic graph and forms one of  $r$ distinct sub-graph. After some algebraic operations we obtain 
\begin{eqnarray}\nonumber
\label{equation.expendeddelta}
\bigtriangleup q_i(k)&&\nonumber\\
&& = E \left[ q_i(k+1)-q_i(k)|\myvec{q}(k)\right] \nonumber \\
&& =\prod_{e_{(m,n)} \in s_i}E \left[ p_n^m(k+1)|p_n^m(k)\right]-\prod_{e_{(m,n)} \in s_i}p_n^m(k) \nonumber  \\
&& \geq \prod_{e_{(m,n)} \in s_i} \left( E \left[ p_n^m(k+1)|p_n^m(k)\right]-p_n^m(k) \right)=\prod_{e_{(m,n)}}\bigtriangleup p_n^m(k) 
\end{eqnarray}

where 
\begin{equation}
\label{equation.delatyeharyal}
 \bigtriangleup  p_u^t(k)_{\forall e_{(t,u) }\in s_i }=ap_u^t(k)\sum_{s \neq u}^{r_i}p_s^t(k)(c_s^t(k)-c_u^t(k))
\end{equation}
Where $p^m(k)$ denotes the action selection probability of learning automaton $A_m$ in the eDLA and $p_n^m(k)$ denotes the probability of selecting edge $e_{(m,n)}$  at stage $k$. In addition $c_n^m(k)$ and $d_n^m(k)=1-c_n^m(k)$ denotes the probability of penalizing and rewarding edge $e_{(m,n)}$  respectively.
\\For all $t,u$  we have $0 < p_u^t(k) < 1$. Since we assumed that edge $e_{(t,u)} \in s_i$ from centeral limit therom for large values of $k$ we conclude that $(c_v^t(k)-c_u^t(k))>0$ $  (\forall v \neq u : \    e_{(t,v)} \notin s_i$). Therefore for large values of $k$, the right hand side of the Equation(\ref{equation.delatyeharyal}) consists of non-negative quantities that imply the right hand side of the Equation(\ref{equation.expendeddelta}) is non-negative that completes the proof of the Lemma\ref{lemma.nonnegativeboodan}
\end{proof}
\begin{qed}
\end{qed}


\begin{corollary}
\label{corollary.absorbingq}
the only absorbing states of Markov process $\{q_i(k)\}$ are 0 and 1
\end{corollary}
\begin{proof} :
Lemma \ref{lemma.nonnegativeboodan} shows that  $\{\myvec{q}(n)\}$ is a sub-martingale. Since the $\{\myvec{q}(n)\}$ is a non-negative and uniformly bounded, martingale theorems imply that $\lim_{k\rightarrow \infty} q_i(k) = \stackrel {*}{q_i}$ exists with probability one. Further  $\myvec{q}(k+1)=\myvec{q}(k)$with probability one, if and only if $q_i^*=0$ or $q_i^*=1$. Hence $\{0,1\}$  constitutes the absorbing set of Markov process $\{q_i(k)\}$ .
\end{proof}
\begin{qed} \end{qed}
Let $$S_r=\{\myvec{q}(k)|q_i(k) \in [0,1]; \ \sum_{i=1}^r{q_i(k)}=1\}$$ and $$S_r^0=\{\myvec{q}(k)|q_i(k) \in (0,1); \ \sum_{i=1}^r{q_i(k)}=1\}$$ Corollary(\ref{corollary.absorbingq}) implies that the set of all unit vectors in $S_r - S_r^0$ forms the set of all absorbing states of Markov process $\{q_i(k)\}$ \\
In the following we will show that under some conditions for the optimum sub graph $s_i$  we have $q_i^*=1$ . Our method is very similar to the given methods in \cite{Norman1968} and \cite{H.Beigy2006}.\\
Assume that $\myvec{q}_i^* \in V_r $ denotes the state to which $\{q_i(k)\}$ converges where $V_r=\{\myvec{e_1},\myvec{e_2},\ldots,\myvec{e_r}\}$  denotes the set of all absorbing states for process $\{q_i(k)\}$. Define $\Gamma_i[\myvec{q}] \equiv Prob\left[ \myvec{q^*}=\myvec{e_i}|\myvec{q}(0)=\myvec{q}\right]$. In fact $\Gamma_i[\myvec{q}]$  denotes the probability of convergence of proposed algorithm to unit vector $\myvec{e_i} \in V_r $  with initial probability vector $\myvec{q}$ . Assume that $D(S_r):S_r \rightarrow\Re$ is the class of all differentiable functions with bounded derivation on $S_r$. Such functions are necessarily continuous. If $\Psi(.) \in D $ the operator $U\Psi(\myvec{q})=E[ \Psi(\myvec{q}(k+1))|\Psi(\myvec{q}(k))=\Psi(\myvec{q})] $ where $E[.]$ represents the mathematical expectation, is linear and positive\cite{Norman1968}\cite{Lakshmivarahan1976} \\
It has been shown that the $\Gamma_i[q]$ is the only continuous solution of $U\Gamma_i[\myvec{q}]=\Gamma_i[\myvec{q}]$  that satisfies the following boundary conditions\cite{Lakshmivarahan1976}\cite{Norman1968}\cite{S.1989}
\begin{equation}
\label{equation.boundaryconditions}
\Gamma_i[\myvec{e_i}]=
\begin{cases}
1& i=j\\
0& i\neq j
\end{cases}
\end{equation} 
However the functional equation $U\Gamma_i[\myvec{q}]=\Gamma_i[\myvec{q}]$   is extremely difficult to solve. Hence the next best thing that can be done is to establish upper and lower bounds onU $\Gamma_i[\myvec{q}]$ . These can be computed by finding two functions $\Phi_1(.)$ and $\Phi_2(.)$ such that 
\begin{eqnarray}
\label{equation.2UDfunction}
U\Phi_1(p)\geq \Phi_1(p) \nonumber \\
U\Phi_2(p)\leq \Phi_2(p) 
\end{eqnarray}
For all $p \in [0,1]$ with appropriate boundary conditions. Inequalities(\ref{equation.2UDfunction}) imply that $\Phi_2(p) \leq \Gamma_i(p) \leq \Phi_1(p)$. Functions $\Phi_1(.)$ and  are $\Phi_2(.)$ called sub-regular and super-regular respectively. It has been shown that these functions can have the following forms:
\begin{equation}
\Phi_i(p)=\frac{e^{x_ip}-1}{e^{x_i}-1} \ i=1,2
\end{equation}
From lemma(\ref{lemma.nonnegativeboodan}) we know that if $q_i^*$ denotes the probability of selecting optimal sub-graph $s_i$ by eDLA then $q_i^* \in {0,1}$ . Define 
\begin{equation}
\label{equation.phidefinition}
\Phi_i[x,\myvec{q}]=\frac{e^{\sfrac{-xq_i}{a}}-1}{e^{\sfrac{-x}{a}}-1} =\frac{\Theta_i[x,\myvec{q}]-1}{\Theta_i[x,1]-1}
\end{equation}
Where $x >0 $ is to be chosen. $\Phi_i[x,\myvec{q}] \in D(S_r)$satisfies boundary conditions \ref{equation.boundaryconditions}. In the last stage of proof we show that $\Phi_i[x,\myvec{q}]$ is sub-regular, thus $\Phi_i[x,\myvec{q}]$ qualifies a lower bound on $\Gamma_i[\myvec{q}]$ . We now determine conditions under which $\Phi_i[x,\myvec{q}]$ is sub-regular. It is easy to see that the classes of super-and sub-regular are closed under addition and multiplication by non-negative constants. Further, the constant functions are regular; hence both super- and sub-regular. For any $x >0 $ ,$\Theta_i[x,1] >\Theta_i[x,0]=1 $  , therefor $\Phi_i[x,\myvec{q}]$ is super-regular or sub-regular if $\Theta_i[x,\myvec{q}]$ . Most of our effort in the rest goes into the proof of this lemma 

\begin{lemma}:
\label{lemma.subregular}
Let  $s_i$ denotes the optimum sub-graph and $\Theta_i[x,\myvec{q}]=e^{\sfrac{-xq_i}{a}}$ where $q_i=\prod_{e_{(m,n)\in s_i}}p_n^m$ and $a$ represent probability of selecting $s_i$ by eDLA and learning rate of eDLA respectively. There is a positive $x$  such that $\Theta_i[x,\myvec{q}]$ is sub-regular
\end{lemma}
\begin{proof}:
From definition of  $U$ we have 
\begin{equation}
\label{equation.lemmesubregular:1}
U\Theta_i[x,\myvec{q}]=E[e^{\sfrac{-xq_i(k+1)}{a}}|\myvec{q}(k)=\myvec{q}]
\end{equation}
From definition of mathematical expectation and Equation(\ref{equation.lemmesubregular:1}) we have: 
\begin{eqnarray}
\label{equation.lemmesubregular:2}
&& E[e^{\sfrac{-xq_i(k+1)}{a}}|\myvec{q}(k)=\myvec{q}]\nonumber \\
&&=\sum_{j=1}^{r}E[e^{\sfrac{-xq_i(k+1)}{a}}|\myvec{q}(k)=\myvec{q},s_j]\times prob[s_j|\myvec{q}(k)] \nonumber \\
&&=\left\lbrace\sum_{j=1}^{r}{d_j(k)e^{\sfrac{-xq_i(k+1)}{a}}}   +\sum_{j=1}^{r}{(1-d_j(k))e^{\sfrac{-xq_i(k)}{a}}} \right\rbrace q_j(k) 
\end{eqnarray}
For the sake of simplicity we show $e^x$ by $exp(x)$ and show $q_j(k)$ and $d_j(k)$ by $q_j$ and $d_j$ respectively. By some algebraic simplification of Equation(\ref{equation.lemmesubregular:3}) we obtain Equation(\ref{equation.lemmesubregular:4}) and Equation(\ref{equation.lemmesubregular:2}) 

\begin{eqnarray}
\label{equation.lemmesubregular:3}
&& \sum_{j=1}^{r}exp\{-\frac{x}{a}q_i(k+1)\} d_j   q_j \nonumber \\
&& = exp\left\lbrace\frac{-x}{a}\prod_{e_{(m,n)} \in s_i}{\{p_n^m+a(1-p_n^m)\}}\right\rbrace q_j  d_j  \nonumber \\
&& +\sum_{\substack{j=1\\j \neq i}}^{r}{exp \left\lbrace \frac{-x}{a} \prod_{\substack{e_{(m,n)}\in s_i \\e_{(m,n)} \in s_j}}\{p_n^m+a(1-p_n^m)\}\prod_{\substack{e_{(m,n)} \in s_i\\e_{(m,n)} \notin s_j}}\{p_n^m(1-a)\} \right\rbrace q_j d_j \,  }
\end{eqnarray}

and
\begin{eqnarray}
\label{equation.lemmesubregular:4}
&& \sum_{j=1}^{r} {exp\left\lbrace\frac{-x}{a}q_i(k) \right\rbrace q_j  (1-d_j)} \nonumber \\
 && =\Theta_i[x,\myvec{q}]\sum_{j=1}^r{q_j (1-d_j)} =\Theta_i[x,\myvec{q}]-\Theta_i[x,\myvec{q}]\sum_{j=1}^r{q_j d_j} 
\end{eqnarray}

In Equation \ref{equation.lemmesubregular:3} we have :
\begin{eqnarray}
\label{equation.lemmesubregular:5}
&&\prod\limits_{\substack{e_{(m,n)} \in s_i \\ e_{(m,n)} \notin s_j}}\{p_n^m(1-a)\}=\nonumber \\
&& \frac{\prod\limits_{\substack{e_{(m,n)} \in s_i \\ e_{(m,n)} \notin s_j}}\{p_n^m(1-a)\} \prod\limits_{\substack{e_{(m,n)} \in s_i \\ e_{(m,n)} \in s_j}}\{p_n^m(1-a)\}}{\prod\limits_{\substack{e_{(m,n)} \in s_i \\ e_{(m,n)} \in s_j}}\{p_n^m(1-a)\}}=\frac{\prod\limits_{\substack{e_{(m,n)} \in s_i }}\{p_n^m(1-a)\}}{\prod\limits_{\substack{e_{(m,n)} \in s_i \\ e_{(m,n)} \in s_j}}\{p_n^m(1-a)\}}
\end{eqnarray}
By considering that $q_i=\prod_{e_{(m,n)} \in s_i}p_n^m$ and equations \ref{equation.lemmesubregular:2} through \ref{equation.lemmesubregular:5} we have:
\begin{eqnarray}
\label{equation.lemmesubregular:6}
&& U\Theta_i[x,\myvec{q}]=q_id_i exp\left\lbrace -\frac{x}{a}(q_i+a(1-q_i)) \right\rbrace \nonumber \\
&&+\sum_{\substack{j=1\\j \neq i}}^{r}{q_j d_jexp \left\lbrace -\frac{x}{a}q_i(1-a)\frac{\prod\limits_{\substack{e_{(m,n)} \in s_i \\ e_{(m,n)} \in s_j}}\{p_n^m(1-a)+a\}}{\prod\limits_{\substack{e_{(m,n)} \in s_i \\ e_{(m,n)} \in s_j}}\{p_n^m(1-a)\}}  \right\rbrace  \,  } \nonumber \\
&&+\Theta_i[x,\myvec{q}] -\Theta_i[x,\myvec{q}]\sum_{j=1}^{r}{q_jd_j}
\end{eqnarray}
It is clear that 
\begin{equation}
\label{equation.lemmesubregular:7}
\frac{\prod_{\substack{e_{(m,n)} \in s_i \\ e_{(m,n)} \in s_j}}\{p_n^m(1-a)+a\}}{\prod_{\substack{e_{(m,n)} \in s_i \\ e_{(m,n)} \in s_j}}\{p_n^m(1-a)\}}=\prod_{\substack{e_{(m,n)} \in s_i \\ e_{(m,n)} \in s_j}}{\frac{p_n^m(1-a)+a}{p_n^m(1-a)}}\geq 1
\end{equation}
By considering the Equation(\ref{equation.lemmesubregular:7}) and the fact that $e^{-x}$ is decreasing function (for $x>0$) we obtain 
\begin{equation}
\label{equation.lemmesubregular:8}
exp \left\lbrace -\frac{x}{a}q_i(1-a)  \frac{\prod_{\substack{e_{(m,n)} \in s_i \\ e_{(m,n)} \in s_j}}\{p_n^m(1-a)+a\}}{\prod_{\substack{e_{(m,n)} \in s_i \\ e_{(m,n)} \in s_j}}\{p_n^m(1-a)\}} \right\rbrace\leq exp \{ -\frac{x}{a}q_i(1-a)\}
\end{equation}
From Equation(\ref{equation.lemmesubregular:7}) and Inequality(\ref{equation.lemmesubregular:8}) we obtain Equation(\ref{equation.lemmesubregular:9}) 
\begin{eqnarray}
\label{equation.lemmesubregular:9}
&& U\Theta_i[x,\myvec{q}] -\Theta_i[x,\myvec{q}]\leq q_id_i exp\left\lbrace -\frac{x}{a}\left(q_i+a(1-q_i)\right) \right\rbrace 
\nonumber \\
&&+\sum_{\substack{j=1\\ j\neq i}}^{r}{q_j d_j exp \left\lbrace -\frac{x}{a}q_i(1-a)  \right\rbrace}  -\Theta_i[x,\myvec{q}]\sum_{j=1}^{r}{q_j d_j}
\end{eqnarray}
But we have
$\Theta_i[x,\myvec{q}]\sum_{j=1}^{r}{q_jd_j}=q_i d_i \Theta_i[x,\myvec{q}] +\Theta_i[x,\myvec{q}]\sum_{\substack{j=1\\j \neq i}}^{r}{q_jd_j}$ and hence:
\begin{eqnarray}
\label{equation.lemmesubregular:10}
&& U\Theta_i[x,\myvec{q}] -\Theta_i[x,\myvec{q}]\leq \nonumber \\
&& q_id_i \left\lbrace exp\{ -\frac{x}{a}(q_i+a(1-q_i)) \}-\Theta_i[x,\myvec{q}]\right\rbrace \nonumber \\
&&+\sum_{\substack{j=1\\ j \neq i}}^{r}{q_j d_j \left\lbrace exp \{ -\frac{x}{a}q_i(1-a)\} -\Theta_i[x,\myvec{q}] \right\rbrace } \nonumber \\
&& =\Theta_i[x,\myvec{q}] \left\lbrace   q_id_i \left\lbrace exp \{-x(1-q_i)\} -1\right\rbrace  + \sum_{\substack{j=1\\j \neq i}}^{r}{q_j d_j \left\lbrace exp \{ xq_i\}-1\right\rbrace }   \right\rbrace \nonumber \\
&& =-xq_i\Theta_i[x,\myvec{q}] \nonumber \\
&& \left\lbrace d_i(1-q_i)\frac{exp\{-x(1-q_i)\} -1}{-x(1-q_i)} - \sum_{\substack{j=1\\j \neq i}}^{r}{q_j d_j \frac{exp \{ xq_i\}-1}{xq_i} }\right\rbrace
\end{eqnarray}
We define 
\begin{equation}
\label{equation.lemmesubregular:11}
V[x]=
\begin{cases}
\frac{e^x-1}{x} & x \neq 0 \\
1& x=0
\end{cases}
\end{equation}
By replacing $V[x]$ into Equation(\ref{equation.lemmesubregular:10}) we have: 
\begin{eqnarray}
\label{equation.lemmesubregular:12}
&& U\Theta_i[x,\myvec{q}] -\Theta_i[x,\myvec{q}]\leq \nonumber \\
&& -xq_i\Theta_i[x,\myvec{q}]\left\lbrace d_i(1-q_i)V[-x(1-q_i)] - \sum_{\substack{j=1\\j \neq i}}^{r}{q_j d_j V[xq_i]}\right\rbrace
\end{eqnarray}
From Equation(\ref{equation.lemmesubregular:12}) we conclude that $\Theta_i[x,\myvec{q}]$ is sub-regular if and only if :
$$\{d_i(1-q_i)V[-x(1-q_i)] - \sum_{\substack{j=1\\j \neq i}}^{r}{q_j d_j V[xq_i]}\}\leq 0$$ that implies:
\begin{equation}
\label{equation.lemmesubregular:11}
\frac{V[-x(1-q_i)]}{V[xq_i]} \leq \frac{\sum_{\substack{j=1\\j \neq i}}^{r}{q_j d_j }}{(1-q_i)d_i}=\frac{\sum_{\substack{j=1\\j \neq i}}^{r}{q_j d_j }}{\sum_{\substack{j=1\\j \neq i}}^{r}{q_jd_i}}=\frac{\sum_{\substack{j=1\\j \neq i}}^{r}{q_j \frac {d_j}{d_i} }}{\sum_{\substack{j=1\\j \neq i}}^{r}{q_j}}
\end{equation}
But:
\begin{equation}
\label{equation.lemmesubregular:12}
min_{j \neq i}{\frac{d_j}{d_i}}\leq \frac{\sum_{\substack{j=1\\j \neq i}}^{r}{q_j \frac {d_j}{d_i} }}{\sum_{\substack{j=1\\j \neq i}}^{r}{q_j}} \leq max_{j \neq i}{\frac{d_j}{d_i}}
\end{equation}
It follows that $\Theta_i[x,\myvec{q}]$ is sub-regular if:
\begin{equation}
\label{equation.lemmesubregular:13}
 \frac{V[-x(1-q_i)]}{V[xq_i]}  \leq max_{j \neq i}{\frac{d_j}{d_i}}
\end{equation}
It can be shown that \cite{Norman1968} 
\begin{equation}
\label{equation.lemmesubregular:14}
\frac{1}{V[x]} \leq \frac{V[-x(1-q_i)]}{V[xq_i]}  
\end{equation}
From Equation(\ref{equation.lemmesubregular:13}) and Equation(\ref{equation.lemmesubregular:14}) we conclude that $\Theta_i[x,\myvec{q}]$ is sub-regular if:
\begin{equation}
\label{equation.lemmesubregular:15}
\frac{1}{V[x]} \leq max_{j \neq i}{\frac{d_j}{d_i}}  \leq 1
\end{equation}
Since $V[x]$ is continuous and strictly monotonically increasing with $V[0]=1$ , a value of $x=x^*$ exists such that if $d_i > d_j (\forall j \neq i $ then $\frac{1}{V[x]} = max_{j \neq i}{\frac{d_j}{d_i}}$ is true. For all $x \in (0,x^*]$ inequality $\frac{1}{V[x]} \leq max_{j \neq i}{\frac{d_j}{d_i}}  \leq 1$ is hold that yields Equation(\ref{equation.lemmesubregular:12})  is true and consequently $\Theta_i[x,\myvec{q}]$ is a sub-regular function satisfying boundary conditions given in Equation(\ref{equation.boundaryconditions}).
 \\We conclude that $\Phi_i[x,\myvec{q}]\leq \Gamma_i[\myvec{q}]\leq 1$.\\
By considering the definition of $\Phi_i[x,\myvec{q}]$  in (\ref{equation.phidefinition}) we see that for any given $\epsilon > 0$  there exists a positive constant $a^* <1 $  as learning rate of eDLA so that the inequality $1-\epsilon \leq \Phi_i[x,\myvec{q}]\leq \Gamma_i[\myvec{q}]=Prob[\myvec{q}^*=\myvec{e_i}|\myvec{q}(0)=\myvec{q}]\leq 1$  is hold for all positive $a<a^*$. This completes the proof of theorem(\ref{theory.mainTheorem}) that 
$\lim\limits_{n \rightarrow \infty}  q_i(n)\stackrel{a.s.}{=}1$
\end{proof}
\begin{qed} \end{qed}
\end{proof}
In the next theorem, the efficiency of the proposed eDLA-based and standard sampling methods  has been compared. Suppose that $n_{eDLA}$ and $n_{standar}$ represent the average number of samples in eDLA-based and standard sampling method respectively. in the context of finding optimal stochastic sub-graph we define accuracy rate of an algorithm as the probability that the result of algorithm is to be optimal solution. 
\begin{theorem}
\label{theorem.numberofsamples}
For a certain sufficiently large accuracy rate, $n_{standar}>n_{eDLA}$
\end{theorem}
\begin{proof}
The proof of this theorem is done in three stages: in the first stage we compute average number of required samples to the number of samples from optimal sub-graph as a function of the probability of choosing optimal sub-graph. in the next stage we show that this function is a decreasing function and finally the average number of the samples in eDLA-based method is compared with same number in the standard sampling method.\\
Suppose that $X_1,X_2,\ldots,X_n$ is a sample from a normal distribution having unknown mean $\mu$ and variance $\sigma^2$. Suppose that we wish to construct a $100(1-\alpha)\%$ confidence interval for $\mu$. Letting $\overline{X}=\sfrac{\sum_{i=1}^{n}{X}}{n}$ and $S^2=\sfrac{\sum_{i=1}^n{(X_i-\overline{X})^2}}{(n-1)}$ denote the mean and sample variance. It has been shown that $\sqrt{n}\sfrac{(\overline{X}-\mu)}{S}$ is a t-random variable with $(n-1)$ degrees of freedom and for any $\alpha \in (0,0.5)$ we have
$$Prob \left\lbrace -t_{\sfrac{\alpha}{2},n-1} < \sqrt{n}\sfrac{(\overline{X}-\mu)}{S} < t_{\sfrac{\alpha}{2},n-1} \right\rbrace=1-\alpha$$
We can say that "with $100(1-\alpha)\%$ confidence"\cite{Ross2004},\cite{Papoulis1991} $$\mu \in (\overline{X}-t_{\sfrac{\alpha}{2},n-1},\overline{X}+t_{\sfrac{\alpha}{2},n-1})$$
Based on the above equations from sampling theory, to obtain a predefined error $d$ the number of required samples should be atleast $n_d$  that we have $d=t_{\sfrac{\alpha}{2},n_d-1}\frac{S}{\sqrt{n_d}}$. This relation shows that for obtain the more accurate interval, more samples are required. This show that, to obtain a certain accuracy for optimal sub-graph in stochastic graph, each edges of this sub-graph to be sampled a certain number of times.
\begin{lemma}
\label{lemma.theorm2.nandnprim}
Suppose that in proposed eDLA-based algorithm we wish to sub-graph $s_i$ to be sampled $n$ times. The proposed algorithm must be run $n^\prime$ times and we have:
\begin{equation} \nonumber
E[n^\prime]=\sum_{t=1}^n{E[\frac{1}{q_i(t)}]}
\end{equation}
 Where $q_i(t)$ represents the probability of choosing $s_i$ by eDLA at time $j$ and $E[.]$ represents the mathematical expectation
\end{lemma}
\begin{proof}
Suppose that the probability of selecting sub-graph $s_i$ by eDLA at time t to be $q_i(t)$. To obtain exactly one sample from $s_i$ the algorithm must be run $n_i(t)$ times and we have:
\begin{equation}
\label{equation.lemma.nandnprim}
E[n_i(t)|\myvec{q}(t)]=\sum_{h=1}^{\infty}{hq_i(t){(1-q_i(t))}^{h-1}}=\frac{1}{q_i(t)}
\end{equation}
by take an expectation from two sides of Equation(\ref{equation.lemma.nandnprim}) we have:
\begin{equation}
\label{equation.lemma.nandnprim:2}
E[E[n_i(t)|\myvec{q}(t)]]=E[n_i(t)]= {E[\frac{1}{q_i(t)}]}
\end{equation}
and finally:
\begin{equation}
n^\prime=\sum_{t=1}^n{E[n_i(t)]}=\sum_{t=1}^n{E[\frac{1}{q_i(t)}]}
\end{equation}
that completes the proof
\end{proof}
\begin{qed} \end{qed}
\begin{remark}: 
From Jensen inequality we know that if X is a random variable and $\varphi$ is a convex function, then $\varphi\left(E\left[X\right]\right) \leq E\left[\varphi(X)\right]$ since $\frac{1}{x}$  is a convex function, we have $E\left[n_i(t)\right]= E\left[ \sfrac{1}{q_i(t)}\right] \geq \sfrac{1}{E\left[q_i(t)\right]}$
\\from lemma(\ref{lemma.nonnegativeboodan}) we know that for optimal sub-graph $s_i$ the inequality $E[q_i(t+1)] \geq E[q_i(t)]$ is true and there for we conclude that $E[n_i(t+1)] \leq E[n_i(t)]$. This inequality shows that the sampling efficiency of eDLA-based proposed algorithm improves with time.
\end{remark}
Let define $\xi(i,n)=\frac{1}{n}\sum_{k=1}^{n}{E[n_i(k)]}$. in next lemma we show the decreasing property of $\xi(i,n)$ for optimal sub-graph $s_i$ 
\begin{lemma}
\label{lemma.theorm2.decreasingproperty}
Suppose that $s_i$ is optimal sub-graph. For any $n_0,n_1 \in N : n_0 \leq n_1 $ we have $$\xi(i,n_1)\leq \xi(i,n_0)$$
\end{lemma}
\begin{proof}:
We show that for any $n \in N$  we have $\xi(i,n+1)\leq \xi(i,n)$. From definition of $\xi(i,n)$ we have:
\begin{eqnarray}
\label{equation.lemmadecreasing:1}
\xi(i,n+1)&=&\frac{1}{n+1}\sum_{k=1}^{n+1}{E[n_i(k)]}=\frac{n\xi(i,n)+E[n_i(n+1)]}{n+1}\nonumber \\
&& =\xi(i,n)+\frac{1}{n+1}\left( E[n_i(n+1)]-\xi(i,n)\right)
\end{eqnarray}
But from previous lemma we know that for optimal sub-graph $s_i$ inequality $E[n_i(t+1)] \leq E[n_i(t)]$ is true
\begin{eqnarray}
\label{equation.lemmadecreasing:2}
\xi(i,n)=\frac{1}{n}\sum_{k=1}^{n}{E[n_i(k)]} \geq \frac{\sum_{k=1}^{n}{E[n_i(n+1)]}}{n}=E[n_i(n+1)]
\end{eqnarray}
From two inequality (\ref{equation.lemmadecreasing:2}) and (\ref{equation.lemmadecreasing:1}) we conclude that $\xi(i,n+1)\leq \xi(i,n)$. that completes the proof.
\end{proof}
\begin{qed} \end{qed}
Suppose that the standard sampling method have got to sample n times from each edges to guarantee a specific accuracy rate for optimal sub-graph that imply $n\times |E|$ edges to be sampled. to guarantee same accuracy, eDLA-based proposed method have to at least $n$ samples from optimal sub-graph. From lemma(\ref{lemma.theorm2.nandnprim}) we know that to obtain n samples from $s_i$ the proposed algorithm have got to run $n\times\xi(i,n)$ times. This implies at most $n\times\xi(i,n)\times (|V|-1)$ edges to be sampled. We have:
$$
\eta=\frac{n_{standard}}{n_{eDLA}} \geq \frac{n\times |E|}{n\times\xi(i,n)\times (|V|-1)} 
$$
Since $\xi(i,n)$ is decreasing function 
$$\exists k \; \forall n>k:\xi(i,n) \leq \frac{|E|}{|V|-1}$$
This complete the proof that for high accuracy rate values  $\eta>1$
\end{proof}
\begin{qed} \end{qed}
\begingroup
\tikzset{every picture/.style={scale=0.78,every picture/.style={}}}

 \section{Experiments:}
To study the feasibility of the proposed algorithm and correctness of theorems ~\ref{theory.mainTheorem} and ~\ref{theorem.numberofsamples} two set of experiments are conducted on some stochastic graphs. Stochastic Shortest Path Problem and Stochastic Minimum Spanning Tree Problem that referred as SSPP and SMSTP respectively. The former is solved by eDLA on directed stochastic graph Graph2 that borrowed from \cite{H.Beigy2006} and the latter is solved on undirected stochastic graph Alex1-a that borrowed from \cite{Hutson}.\\
In \cite{H.Beigy2006} the authors proposed an algorithm based on the Distributed Learning Automata to solve the SSPP. In this method at each stage, source automaton (corresponding to the source node)selects one of its actions (as a sample realization of its action probability vector). The selected action activates another Learning Automata. The process of choosing an action and activating an automaton is repeated until destination automaton is reached. The weight of constructed path is used to reward or penalize the actions selected by Learning Automata in DLA.\\
In \cite{AkbariTorkestani2011} the authors proposed an algorithm based on the colony of Learning Automata to solve SMSTP. In this method, at each time, a node in graph is randomly selected and associative learning automaton selects an action corresponding to an edge of this node  so that prevent the loop construction. Consequently, at the end of each stage the selected edges construct a tree. The weight of constructed tree is used to reward or penalize the action chosen by Learning Automaton. 
It has been shown that the proposed algorithm more efficient than other methods for finding the minimum spanning tree in stochastic graphs\\

 In this section we use an eDLA to solve these problems. At each stage, a Learning Automaton is selected as ROOT. In SSPP case, the ROOT is source automaton and in SMSTP each node can be selected randomly as ROOT. The activity level of ROOT is set to \textit{Fire}('Fi'). 
Each automaton at \textit{Fire} level of activity chooses one its available actions based on the selection action probability vector. In SMSTP the next \textit{Fire} automaton is selected randomly from the set of automata with \textit{Active} level of activity. In SSPP, the current action determines specific learning automata in the set of automata with \textit{Active} level of activity  that should be fired. This process (fire and action) continues until the reaching to destination node (destination automaton downgrades to Off level of activity) or constructing a tree(all of the learning automata downgrade to Off level of activity)in SSPP and SMSTP respectively\\
The process of finding a path or tree is repeated until the  stopping criteria is reached which at this point the last sample path or tree has minimum expected average weight. The algorithm stops if the product of probability of selecting the edge of sampled sub-graph is greater than a predefined threshold or the number of sampled sub-graph is reached to a certain number.
\\The differences between the mechanisms that used by eDLA to solving each of the above problems are described in table ~\ref{table.twomechanisms} 

\begin{table}
  \centering
  \caption{Differences between SSPP and SMSTP solving mechanisms by eDLA}
  \label{table.twomechanisms}
    \begin{tabular}{lll}
    \toprule
   &Problem \\ \cmidrule(l){2-3}
    & \textbf{Stocahstic Shortest Path} & \textbf{Stochastic Minimum Spanning Tree} \\
    \midrule
    \textbf{ROOT?} & Deterministic: S & Stochastic: uniformly random selection \\
    \textbf{FIRE?} & Deterministic & Stochastic \\
    \textbf{ACTIVE?} & Adjacency by \textit{Fire} automaton & Adjacency by \textit{Fire} automaton  \\
    \textbf{Termination?} & D $\stackrel{?}{=}$ Off & All Automata $\stackrel{?}{=}$ Off \\
    \bottomrule
    \end{tabular}%
  \label{tab:addlabel}%
\end{table}%

Graph2 is a directed graph with 15 nodes, 42 arcs, $v_s=1$, $v_d=15$ and optimal path $\pi^*= (1,4,12,14,15)$. Edge cost distribution is given in table~\ref{table.graph2beigy}

\begin{table}
  \centering
  \caption{Weight Distribution of Graph2}
    \label{table.graph2beigy}
    \begin{tabular}{llllllllll}
    \toprule
        \multicolumn{1}{c}{Edge}&&
        \multicolumn{4}{l}{Weights}
       &\multicolumn{4}{c}{Probabilities}\\
    \midrule    
    (1,2) & 19    & 25    & 36    &       &       & 0.6   & 0.3   & 0.1   &  \\
    (1,3) & 21    & 24    & 25    & 29    &       & 0.5   & 0.2   & 0.2   & 0.1 \\
    (1,4) & 11    & 13    & 16    &       &       & 0.4   & 0.4   & 0.2   &  \\
    (2,11) & 24    & 28    & 31    &       &       & 0.5   & 0.3   & 0.2   &  \\
    (2,5) & 21    & 30    &       &       &       & 0.7   & 0.3   &       &  \\
    (3,6) & 18    & 24    &       &       &       & 0.7   & 0.3   &       &  \\
    (2,6) & 13    & 37    & 39    &       &       & 0.6   & 0.2   & 0.2   &  \\
    (2,3) & 11    & 20    & 24    &       &       & 0.6   & 0.3   & 0.1   &  \\
    (3,7) & 23    & 30    & 34    &       &       & 0.4   & 0.3   & 0.3   &  \\
    (3,8) & 14    & 23    & 34    &       &       & 0.5   & 0.4   & 0.1   &  \\
    (3,4) & 22    & 30    &       &       &       & 0.7   & 0.3   &       &  \\
    (4,9) & 35    & 40    &       &       &       & 0.6   & 0.4   &       &  \\
    (4,12) & 16    & 19    & 37    &       &       & 0.5   & 0.4   & 0.1   &  \\
    (5,13) & 28    & 35    & 37    & 40    &       & 0.4   & 0.3   & 0.2   & 0.1 \\
    (5,15) & 29    & 32    &       &       &       & 0.7   & 0.3   &       &  \\
    (5,6) & 18    & 25    & 29    &       &       & 0.5   & 0.3   & 0.2   &  \\
    (5,10) & 27    & 33    & 40    &       &       & 0.4   & 0.3   & 0.3   &  \\
    (5,7) & 15    & 17    & 19    & 26    &       & 0.3   & 0.3   & 0.3   & 0.1 \\
    (6,13) & 21    & 23    &       &       &       & 0.5   & 0.5   &       &  \\
    (6,7) & 12    & 23    & 31    &       &       & 0.5   & 0.3   & 0.2   &  \\
    (7,10) & 19    & 23    & 37    &       &       & 0.6   & 0.2   & 0.2   &  \\
    (7,8) & 12    & 15    & 22    & 24    &       & 0.3   & 0.3   & 0.3   & 0.1 \\
    (8,7) & 14    & 34    & 39    &       &       & 0.6   & 0.2   & 0.2   &  \\
    (7,6) & 12    & 23    & 31    &       &       & 0.5   & 0.3   & 0.2   &  \\
    (8,14) & 14    & 15    & 27    & 32    &       & 0.3   & 0.3   & 0.2   & 0.2 \\
    (8,9) & 13    & 31    & 32    &       &       & 0.8   & 0.1   & 0.1   &  \\
    (4,8) & 13    & 23    & 34    &       &       & 0.4   & 0.3   & 0.3   &  \\
    (7,9) & 10    & 17    & 20    &       &       & 0.6   & 0.3   & 0.1   &  \\
    (9,10) & 16    & 18    & 36    & 39    &       & 0.3   & 0.3   & 0.2   & 0.2 \\
    (9,15) & 12    & 13    & 25    & 32    &       & 0.4   & 0.3   & 0.2   & 0.1 \\
    (9,14) & 19    & 24    & 29    &       &       & 0.4   & 0.3   & 0.3   &  \\
    (10,13) & 14    & 20    & 25    & 32    &       & 0.3   & 0.3   & 0.2   & 0.2 \\
    (10,15) & 15    & 19    & 25    &       &       & 0.4   & 0.3   & 0.3   &  \\
    (10,14) & 23    & 34    &       &       &       & 0.9   & 0.1   &       &  \\
    (11,13) & 13    & 31    & 25    &       &       & 0.6   & 0.3   & 0.1   &  \\
    (5,11) & 18    & 19    & 20    & 23    &       & 0.3   & 0.3   & 0.3   & 0.1 \\
    (6,11) & 10    & 19    & 39    &       &       & 0.5   & 0.4   & 0.1   &  \\
    (8,12) & 15    & 36    & 39    &       &       & 0.5   & 0.3   & 0.2   &  \\
    (9,12) & 16    & 22    &       &       &       & 0.7   & 0.3   &       &  \\
    (12,14) & 10    & 13    & 18    & 34    &       & 0.3   & 0.3   & 0.3   & 0.1 \\
    (13,15) & 12    & 31    &       &       &       & 0.9   & 0.1   &       &  \\
    (14,15) & 14    & 16    & 32    &       &       & 0.5   & 0.3   & 0.2   &  \\
    \bottomrule
    \end{tabular}%
  \label{tab:addlabel}%
\end{table}%

Alex1-a is a undirected graph with 8 nodes, 14 edges and optimal spanning tree $\tau^*=(1,3)(2,3)(2,5)(3,4)(5,6)(6,7)(7,8)$ Edge weight distribution is given in table ~\ref{table.alex1a}

\begin{table}
  \centering
  \caption{Weight Distribution of Alex1-a}
    \label{table.alex1a}
    \begin{tabular}{@{}clllll@{}}
    \toprule
        \multicolumn{1}{c}{Edge}&
        \multicolumn{2}{c}{Weights}&
       &\multicolumn{2}{c}{Probabilities}\\
 
    \midrule
    (1,2) & 70    & 94    &       & 0.95  & 0.05 \\
    (1,3) & 25    & 52    &       & 0.80  & 0.20 \\
    (1,4) & 42    & 61    &       & 0.70  & 0.30 \\
    (2,3) & 15    & 43    &       & 0.80  & 0.20 \\
    (2,5) & 26    & 50    &       & 0.85  & 0.15 \\
    (3,4) & 21    & 68    &       & 0.90  & 0.10 \\
    (3,6) & 65    & 75    &       & 0.60  & 0.40 \\
    (3,7) & 89    & 78    &       & 0.70  & 0.30 \\
    (4,7) & 90    & 96    &       & 0.95  & 0.05 \\
    (4,8) & 89    & 96    &       & 0.85  & 0.15 \\
    (5,6) & 32    & 67    &       & 0.80  & 0.20 \\
    (6,7) & 16    & 42    &       & 0.70  & 0.30 \\
    (7,8) & 3     & 15    &       & 0.60  & 0.40 \\
    (6,8) & 98    & 45    &       & 0.80  & 0.20 \\
    \bottomrule
    \end{tabular}%
  \label{tab:addlabel}%
\end{table}%

In all the simulations that are presented in the rest of this paper, \lri schema is used as learning algorithm for updating the action selection probability vectors. 
\\ In the SSPP, each algorithm (Porposed and DLA-based\cite{H.Beigy2006}) is tested on Graph2. each algorithm is terminated when the number of sampled path is greater than 10000 or probability of selected path is greater than 90\%  The results for each algorithm are summarized in table~\ref{tables.SSPP}. Every value in this table is averaged over 50 runs. To compare the results, 4 metrics are calculated: the average number of samples that to be taken from all edges in the graph(AS). The average number of iterations for a converged(AI), the average required time for a converged run(AT) and percentage of converged runs(PC).
\\ A similar method was applied to solve the SMSTP. The results for each algorithm (proposed and LA-based\cite{AkbariTorkestani2011}) are summarized in table~\ref{tables.SMSTP}.
\\From these tables of the results, the following points can be made:
 \begin{itemize}
 \item[1] The proposed network structure of the learning automata is quite flexible so that it is able to solve different network optimization problems. Furthermore, since the most combinatroial optimization problems (COPs)can be formulated as optimization problems on a weighted
graph, the proposed eDLA structure will also be able to solve combinatorial optimization problem.  
\item[2] Proposed threshold computation improves the rate of convergence. This is due because of the including the variance of obtained solutions in the calculating of the threshold value. The proposed threshold value is more realistic criterion for comparison than the dynamic threshold value that has been proposed in previous works\cite{H.Beigy2006}.
It seems that this variance aware threshold is useful to prevent eDLA from becoming stuck in local optima. The value of thresholds in proposed method and one that is used by ~\cite{H.Beigy2006} in compared to the average path weight of shortest path in Graph2 is illustated in figure ~\ref{figures.wthcmp} 
\begin{figure}
\includegraphics[scale=0.7]{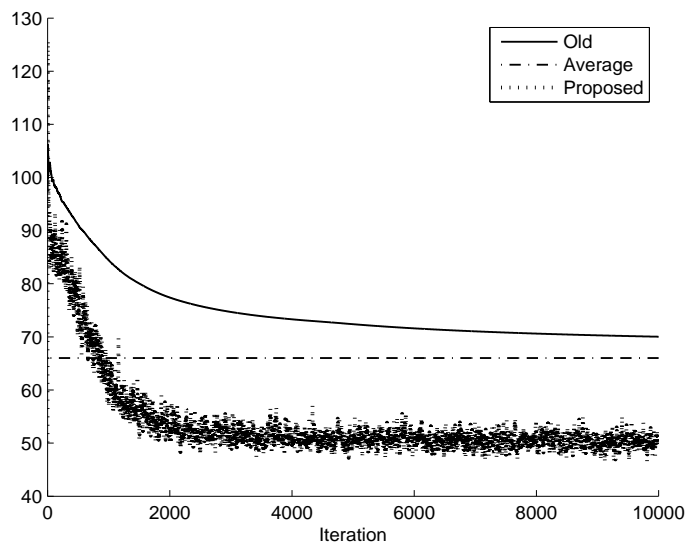} 
\caption{Difference between the values of thresholds used to evaluate eDLA compared to the average path weight in Graph2 }
\label{figures.wthcmp}
\end{figure}

\item[3]The results show that the proposed algorithm can be used as a uniform framework to solve a wide range of stochastic network optimization problems through sampling.
\item[4]Comparing the percentage of convergence (PC column) in tables ~\ref{tables.SMSTP} and ~\ref{tables.SSPP} show that the proposed variance aware threshold value can be to improve the convergence rate of the algorithm. This improvement such that  it appears approximately independent of learning rate value, the convergence is guaranteed.

\begin{table*}
  \centering
  \ra{1.2}
  \caption{The simulation results of SSPP}
  \label{tables.SSPP}
    \begin{tabular}{@{}ccccccccccc@{}}
    \toprule
       \multicolumn{1}{c}{}& \phantom{abc}
       & \multicolumn{4}{c}{DLA baed\cite{H.Beigy2006}}&      \phantom{abc}&\multicolumn{4}{c}{Proposed Algorithm}\\
       \cmidrule{3-6}\cmidrule{8-11}
$\alpha$ &&AS&AI&AT&PC&&AS&AI&AT&PC\\
 \midrule
    0.003 &       & 16825 & 3918  & 0.052 & 100\% &       & 20863 & 4869  & 0.058 & 100\% \\

    0.004 &       & 13244 & 3094  & 0.039 & 100\% &       & 15553 & 3628  & 0.059 & 100\% \\
    0.005 &       & 10125 & 2357  & 0.041 & 100\% &       & 12422 & 2900  & 0.062 & 100\% \\
    0.006 &       & 9377  & 2131  & 0.062 & 100\% &       & 10287 & 2404  & 0.039 & 100\% \\
    0.007 &       & \textbf{7548}  & \textbf{1769}  & \textbf{0.045} & \textbf{100\%} &       & 9050  & 2114  & 0.041 & 100\% \\
    0.008 &       & 6881  & 1614  & 0.036 & 98\%  &       & 7892  & 1844  & 0.055 & 100\% \\
    0.009 &       & 5878  & 1375  & 0.033 & 98\%  &       & 7052  & 1648  & 0.050 & 100\% \\
    0.01  &       & 5485  & 1287  & 0.043 & 98\%  &       & 5968  & 1391  & 0.052 & 100\% \\
    0.02  &       & 3309  & 796   & 0.047 & 90\%  &       & 3193  & 747   & 0.043 & 100\% \\
    0.03  &       & 2127  & 504   & 0.054 & 86\%  &       & 2250  & 531   & 0.053 & 100\% \\
    0.04  &       & 1993  & 475   & 0.070 & 76\%  &       & 1755  & 415   & 0.038 & 100\% \\
    0.05  &       & 1593  & 387   & 0.047 & 76\%  &       & \textbf{1354}  & \textbf{319}   & \textbf{0.038} & \textbf{100\%} \\
    0.06  &       & 1338  & 320   & 0.061 & 68\%  &       & 1339  & 316   & 0.050 & 98\% \\
    0.07  &       & 1394  & 328   & 0.120 & 58\%  &       & 953   & 222   & 0.038 & 100\% \\
    0.08  &       & 1132  & 269   & 0.077 & 60\%  &       & 1060  & 257   & 0.068 & 100\% \\
    0.09  &       & 758   & 180   & 0.062 & 66\%  &       & 1098  & 267   & 0.063 & 96\% \\
    \bottomrule
    \end{tabular}%
  \label{tab:addlabel}%
\end{table*}%

\item[5] In order to compare the algorithms, minimum number of samples and iterations require to guarantee the 100\% converges is to be considered. This cases indicated by bold letters in tables ~\ref{tables.SSPP} and ~\ref{tables.SMSTP}. In the case of SSPP, the results of table~\ref{tables.SSPP} show that the proposed algorithm can reduce the average number of required samples and average number of iterations to $\sfrac{1}{5}$ (80\% improvement). The same matter applies to SMSTP results.

\begin{table}
  \centering
  \caption{The simulation results of SMSTP}
  \label{tables.SMSTP}
    \begin{tabular}{@{}lccccccc@{}}
    \toprule
    \multicolumn{1}{c}{}&  \multicolumn{3}{c}{LA based\cite{AkbariTorkestani2011}}&  &\multicolumn{3}{c}{Proposed eDLA-based}\\
       \cmidrule{2-4}\cmidrule{6-8}
$\alpha$ &AS&AI&PC&&AS&AI&PC\\
    \midrule
    0.007 & 79430 & 11348 & 100\% &       & 96268 & 13753 & 100\% \\
    0.008 & 78068 & 10393 & 100\% &       & 84479 & 12069 & 100\% \\
    0.009 & 61280 & 8755  & 100\% &       & 75895 & 10843 & 100\% \\
    0.01  & \textbf{67330} & \textbf{7747}  & \textbf{100\%} &       & 65389 & 9342  & 100\% \\
    0.02  & 65268 & 4064  & 88\%  &       & 31484 & 4498  & 100\% \\
    0.03  & 98431 & 3456  & 68\%  &       & 18770 & 2682  & 100\% \\
    0.04  & 128830 & 2783  & 60\%  &       & 13650 & 1950  & 100\% \\
    0.05  & 139707 & 2457  & 50\%  &       & 11816 & 1688  & 100\% \\
    0.06  & 191879 & 2103  & 36\%  &       & 9865  & 1409  & 100\% \\
    0.07  & 170472 & 4058  & 16\%  &       & \textbf{7101}  & \textbf{1015}  & \textbf{100\%} \\
    \bottomrule
    \end{tabular}%
  \label{tab:addlabel}%
\end{table}%

\item[6] 
Convergence behavior illustrated by the any-time curve of  algorithm. This property is showed by plotting the development of the probability of the optimal sub-graph in time. To do so, probability of optimal sub-graph at each step of algorithm is defined as the product of probabilities of automata actions corresponding to the edges that constructing optimal sub-graph. The term "any-time" refers to the property that the search can be stopped at any time and the proposed algorithm will have some sub-optimal solution. A point with coordinate $(x,y)$  in any-time curve gives us the average number of iterations $x$  required to obtain the optimal solution with probability of $y$. The "any time" plots show the correctness of theorem \ref{theory.mainTheorem}. Furthermore show some other features of the proposed algorithm.
\\In another experiment, the proposed algorithm is used to finding optimal shortest path in stochastic graph Graph2. The algorithm terminates when the number of sampled path is greater than or equal to 10000. the value of optimal path probability in each iteration is computed and averaged on 10 simulations. The results are shown in figure
~\ref{figures.compareallgraph2} for different learning rate(a). Looking at different curves in figure
~\ref{figures.compareallgraph2} indicated that the average iterations required by the proposed algorithm for finding the optimal solution is dependent on learning rate.
\item[7] The "any-time" curves of the proposed algorithm and DLA-based algorithm \cite{H.Beigy2006} for SSPP has been shown in figures ~\ref{figures.khalilissppall} and ~\ref{figures.beigyall} respectively. For each individual learning rate, the slop of POP in the proposed algorithm is higher than that one in DLA-based algorithm. Further more the any-time curve in the new algorithm is more stable than that on DLA-based (figure \ref{figures.comparesspp}) 
\item[8]The any-time curve of proposed algorithm for SMSTP in Alex1-a is shwon in figure~\ref{figures.khaliliallpopsmstp}. As indicated in that, the speed of converge to optimal path independent on the learning rate. By a low learning rate, algorithm converge slowly but stable to the optimal solution. High learning rates leads to rapid convergence to optimal solution, but may converges to sub-optimal solution.
In figure ~\ref{figures.akbarikhalilicmp} the probability of optimal spanning tree in new proposed eDLA-based and old LA-based are compared. The slop of curve in new method is higher than that one in old. Besides, the new method is more stable than old method is. As a result at each stages a sampled tree in new method is more probable to be optimal spanning tree than that one in old method.

\end{itemize}

\begin{figure}
\input{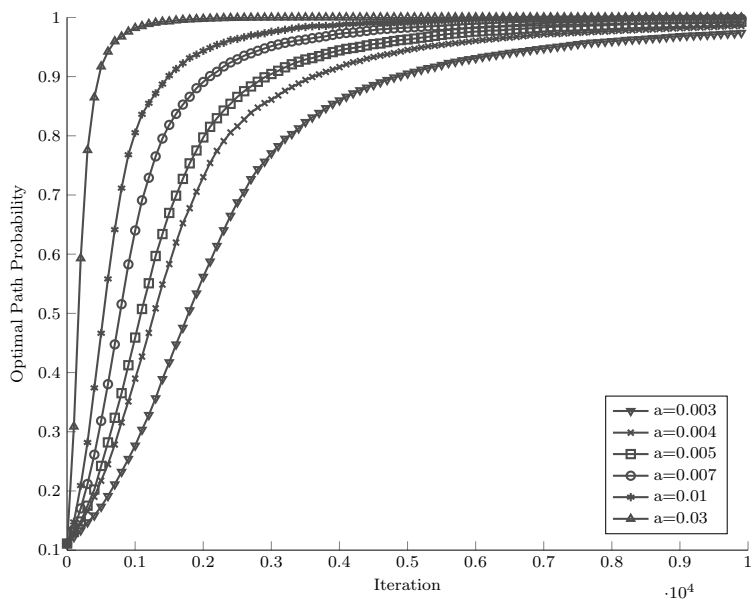}
\caption{Probability of Optimal Path for Graph2 in new eDLA-based proposed method for different learning rates}
\label{figures.compareallgraph2}
\end{figure}

 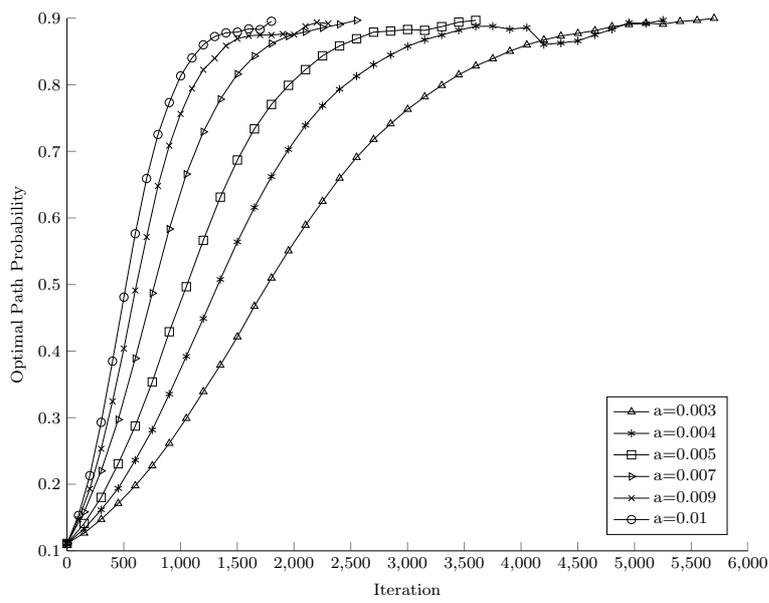
\begin{figure}
%
%
%
%
\begin{tikzpicture}

\begin{axis}[%
width=4.52083333333333in,
height=3.565625in,
scale only axis,
xmin=0,
xmax=6000,
xlabel={Iteration},
ymin=0.1,
ymax=0.9,
ylabel={Optimal Path Probability},
axis x line*=bottom,
axis y line*=left,
legend style={at={(0.97,0.03)},anchor=south east,draw=black,fill=white,legend cell align=left}
]
\addplot [
color=black,
solid,
mark=triangle,
mark options={solid}
]
table[row sep=crcr]{
1 0.111111111111111\\
151 0.126745255644033\\
301 0.14734935413361\\
451 0.171297141788272\\
601 0.197780315981089\\
751 0.227773875091583\\
901 0.261552536188057\\
1051 0.299053768573308\\
1201 0.339263815453853\\
1351 0.37894417267084\\
1501 0.421433151231234\\
1651 0.467279090151664\\
1801 0.509784470239245\\
1951 0.550698720361289\\
2101 0.589097546914585\\
2251 0.624914280745062\\
2401 0.659759497427793\\
2551 0.690963386463071\\
2701 0.717935406868984\\
2851 0.741367658334342\\
3001 0.763143584481626\\
3151 0.781936414302958\\
3301 0.799036719848341\\
3451 0.814976073722672\\
3601 0.828237849637138\\
3751 0.839138435240317\\
3901 0.850328374939038\\
4051 0.859853947444202\\
4201 0.867422441882868\\
4351 0.873186343423007\\
4501 0.877175086239974\\
4651 0.881275674328729\\
4801 0.887345387908807\\
4951 0.890446122285143\\
5101 0.892443381845852\\
5251 0.890981135333009\\
5401 0.895118475679882\\
5551 0.896695103233126\\
5701 0.899839183797501\\
};
\addlegendentry{a=0.003};

\addplot [
color=black,
solid,
mark=asterisk,
mark options={solid}
]
table[row sep=crcr]{
1 0.111111111111111\\
151 0.133119052893697\\
301 0.161732732104448\\
451 0.193913842568997\\
601 0.236265358020877\\
751 0.281571987418411\\
901 0.335417890635162\\
1051 0.392074942891154\\
1201 0.449027375433373\\
1351 0.507643636211228\\
1501 0.56407638400378\\
1651 0.615720766787878\\
1801 0.662094889297385\\
1951 0.702786974229208\\
2101 0.739105313366109\\
2251 0.768583540262204\\
2401 0.793272275268772\\
2551 0.812801417478482\\
2701 0.830138094055542\\
2851 0.844796055324976\\
3001 0.857786061122482\\
3151 0.86758428170383\\
3301 0.874672825486235\\
3451 0.881946077485603\\
3601 0.887630980193465\\
3751 0.88796342804656\\
3901 0.883303736198954\\
4051 0.885963971390622\\
4201 0.860437378966888\\
4351 0.862686829842028\\
4501 0.865840005958065\\
4651 0.874984365708313\\
4801 0.883265935866594\\
4951 0.893275965311343\\
5101 0.891918038136083\\
5251 0.897125457900381\\
};
\addlegendentry{a=0.004};

\addplot [
color=black,
solid,
mark=square,
mark options={solid}
]
table[row sep=crcr]{
1 0.111111111111111\\
151 0.141067708022565\\
301 0.180350755045871\\
451 0.230771839052609\\
601 0.287520147432723\\
751 0.353627525813265\\
901 0.429054513819868\\
1051 0.496703007334643\\
1201 0.566317533994283\\
1351 0.631174098524308\\
1501 0.686943908754168\\
1651 0.73375307876274\\
1801 0.770446376368616\\
1951 0.799143249846866\\
2101 0.82272114409843\\
2251 0.843460439447641\\
2401 0.858354448321872\\
2551 0.869065238993185\\
2701 0.879015672521714\\
2851 0.88067008226703\\
3001 0.882978145055423\\
3151 0.88170824361856\\
3301 0.887254756968228\\
3451 0.894085612787493\\
3601 0.896657468227269\\
};
\addlegendentry{a=0.005};

\addplot [
color=black,
solid,
mark=triangle,
mark options={solid,,rotate=270}
]
table[row sep=crcr]{
1 0.111111111111111\\
151 0.158471315196613\\
301 0.220132011196718\\
451 0.297081265276968\\
601 0.389080465500541\\
751 0.487026459977872\\
901 0.583376211727544\\
1051 0.665752362677408\\
1201 0.729309913328732\\
1351 0.778274766389059\\
1501 0.816069209405263\\
1651 0.843117188654758\\
1801 0.861917463590882\\
1951 0.873063387233643\\
2101 0.879590412899002\\
2251 0.88545126314465\\
2401 0.890713218575747\\
2551 0.89675440750833\\
};
\addlegendentry{a=0.007};

\addplot [
color=black,
solid,
mark=x,
mark options={solid}
]
table[row sep=crcr]{
1 0.111111111111111\\
101 0.147504952985386\\
201 0.193047187402596\\
301 0.253360976833231\\
401 0.324591297256041\\
501 0.403809043566565\\
601 0.491138845325931\\
701 0.57136493794287\\
801 0.647995117162029\\
901 0.708542401360565\\
1001 0.756201939400039\\
1101 0.794312407406115\\
1201 0.822483969407978\\
1301 0.839734117358284\\
1401 0.858635583556846\\
1501 0.869162855505603\\
1601 0.873415384254304\\
1701 0.874706497647292\\
1801 0.874646894885086\\
1901 0.876227912695523\\
2001 0.875324965757099\\
2101 0.887772920001134\\
2201 0.893680842064123\\
2301 0.892030793816423\\
};
\addlegendentry{a=0.009};

\addplot [
color=black,
solid,
mark=o,
mark options={solid}
]
table[row sep=crcr]{
1 0.111111111111111\\
101 0.153214606700506\\
201 0.21308878743762\\
301 0.293103767908358\\
401 0.384831091495304\\
501 0.481078875789803\\
601 0.576323061290772\\
701 0.659177561951401\\
801 0.725270905630365\\
901 0.773235150047161\\
1001 0.813294947993423\\
1101 0.840227814574796\\
1201 0.859996392915117\\
1301 0.872422372995182\\
1401 0.878004637776836\\
1501 0.879107163590698\\
1601 0.883997394797421\\
1701 0.882950288125496\\
1801 0.895387125734983\\
};
\addlegendentry{a=0.01};

\end{axis}
\end{tikzpicture}%
  \caption{Probability of optimal path in eDLA-based algorithm for different learning rates in Graph2}
\label{figures.khalilissppall} 
 \end{figure}

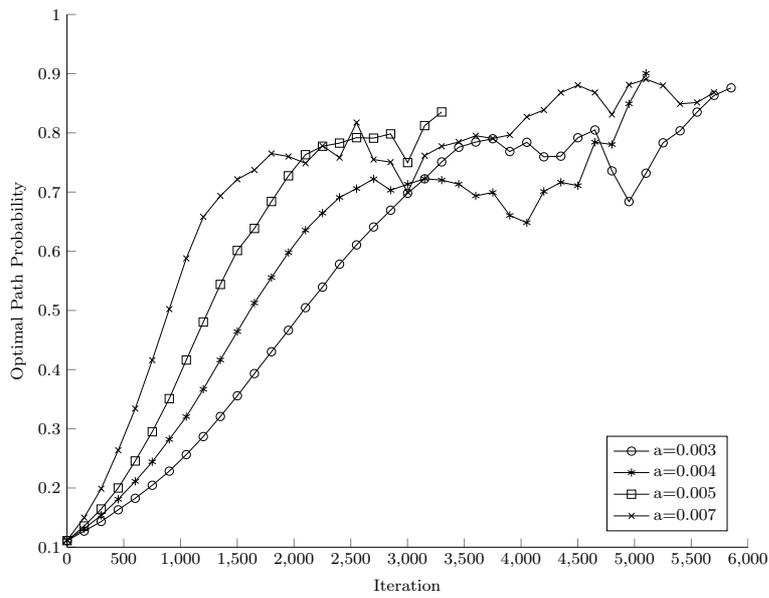
\begin{figure}
%
%
%
%
\begin{tikzpicture}

\begin{axis}[%
width=4.52083333333333in,
height=3.565625in,
scale only axis,
xmin=0,
xmax=6000,
xlabel={Iteration},
ymin=0.1,
ymax=1,
ylabel={Optimal Path Probability},
axis x line*=bottom,
axis y line*=left,
legend style={at={(0.97,0.03)},anchor=south east,draw=black,fill=white,legend cell align=left}
]
\addplot [
color=black,
solid,
mark=o,
mark options={solid}
]
table[row sep=crcr]{
1 0.111151050712311\\
151 0.127150740043951\\
301 0.143422997290899\\
451 0.163171051881058\\
601 0.1826184847844\\
751 0.204596362624235\\
901 0.228444141967848\\
1051 0.256460047159602\\
1201 0.286933161351528\\
1351 0.320836558291256\\
1501 0.355861460797147\\
1651 0.39331199774696\\
1801 0.430193765115312\\
1951 0.466496903047999\\
2101 0.504675233765267\\
2251 0.539387830583106\\
2401 0.5780671097534\\
2551 0.610694229591995\\
2701 0.641051452963876\\
2851 0.669307492312202\\
3001 0.697931684269331\\
3151 0.722355820080857\\
3301 0.750871994711866\\
3451 0.775657132512436\\
3601 0.784450342018595\\
3751 0.790095366409968\\
3901 0.768440277141001\\
4051 0.784035229702247\\
4201 0.75953888922944\\
4351 0.760639321046932\\
4501 0.79180068608314\\
4651 0.804909579499304\\
4801 0.735762869218844\\
4951 0.683846992276683\\
5101 0.731891312297916\\
5251 0.783180428345404\\
5401 0.803524601520081\\
5551 0.835097237376031\\
5701 0.863302314275531\\
5851 0.876043648682361\\
};
\addlegendentry{a=0.003};

\addplot [
color=black,
solid,
mark=asterisk,
mark options={solid}
]
table[row sep=crcr]{
1 0.111316335717689\\
151 0.131361469123874\\
301 0.153825552142985\\
451 0.181068082625421\\
601 0.211285357516812\\
751 0.2442695321808\\
901 0.282133877672162\\
1051 0.32055823167004\\
1201 0.366869094361948\\
1351 0.416287916021344\\
1501 0.464387749186623\\
1651 0.512599347390167\\
1801 0.555621109714809\\
1951 0.597717385380754\\
2101 0.635681001608312\\
2251 0.664522883999642\\
2401 0.690953021504228\\
2551 0.705877940506134\\
2701 0.722176565341722\\
2851 0.703086217181922\\
3001 0.71327136703445\\
3151 0.722470944692786\\
3301 0.720203238415804\\
3451 0.713353542352343\\
3601 0.693553596691929\\
3751 0.699318554595586\\
3901 0.660443795490703\\
4051 0.648260385814048\\
4201 0.700964294297729\\
4351 0.71647652792325\\
4501 0.710816535296209\\
4651 0.784247351583061\\
4801 0.780380304085207\\
4951 0.849178775248606\\
5101 0.900576782155514\\
};
\addlegendentry{a=0.004};

\addplot [
color=black,
solid,
mark=square,
mark options={solid}
]
table[row sep=crcr]{
1 0.111122555437778\\
151 0.135781587222676\\
301 0.164479767178182\\
451 0.199999146772444\\
601 0.245650919584232\\
751 0.295088655847718\\
901 0.351039256350365\\
1051 0.416247783946693\\
1201 0.480477233732097\\
1351 0.543927348479769\\
1501 0.601293378722037\\
1651 0.638516470299023\\
1801 0.684085392208324\\
1951 0.727473944919661\\
2101 0.763085955494149\\
2251 0.777493783671366\\
2401 0.782723270492316\\
2551 0.791916982761852\\
2701 0.791117238385525\\
2851 0.798283708168474\\
3001 0.749807804203577\\
3151 0.812322526228054\\
3301 0.835314779474366\\
};
\addlegendentry{a=0.005};

\addplot [
color=black,
solid,
mark=x,
mark options={solid}
]
table[row sep=crcr]{
1 0.111597248432711\\
151 0.150058109537526\\
301 0.198814811887886\\
451 0.263779955157239\\
601 0.334181369191022\\
751 0.415965625304219\\
901 0.502273873669193\\
1051 0.587969684881596\\
1201 0.657979256557475\\
1351 0.693467890938976\\
1501 0.721462718218961\\
1651 0.737201784596216\\
1801 0.765182968061237\\
1951 0.760039170607352\\
2101 0.748667894823392\\
2251 0.778095110326089\\
2401 0.75815248068656\\
2551 0.81743583481822\\
2701 0.754901774265192\\
2851 0.75056100056742\\
3001 0.699763917874202\\
3151 0.761484623543995\\
3301 0.777306640258125\\
3451 0.784600892722871\\
3601 0.795181694123591\\
3751 0.790596278351171\\
3901 0.796422413470065\\
4051 0.827029765636996\\
4201 0.838594361829223\\
4351 0.868052595322911\\
4501 0.880484525760339\\
4651 0.868185277295701\\
4801 0.830999487942538\\
4951 0.88134396983964\\
5101 0.890590181124346\\
5251 0.880038912089782\\
5401 0.848970232284687\\
5551 0.851409494439011\\
5701 0.868305612677393\\
};
\addlegendentry{a=0.007};

\end{axis}
\end{tikzpicture}%
\caption{Probability of optimal path in DLA-based algorithm for different learning rates in Graph2}
\label{figures.beigyall}
\end{figure}
\begin{figure}
%
%
%
%
\begin{tikzpicture}

\begin{axis}[%
width=4.52083333333333in,
height=3.565625in,
scale only axis,
xmin=0,
xmax=6000,
xlabel={Iteration},
ymin=0.1,
ymax=0.9,
ylabel={Optimal Path Probability},
axis x line*=bottom,
axis y line*=left,
legend style={at={(0.97,0.03)},anchor=south east,draw=black,fill=white,legend cell align=left}
]
\addplot [
color=black,
solid,
mark=o,
mark options={solid}
]
table[row sep=crcr]{
1 0.111111111111111\\
151 0.126745255644033\\
301 0.14734935413361\\
451 0.171297141788272\\
601 0.197780315981089\\
751 0.227773875091583\\
901 0.261552536188057\\
1051 0.299053768573308\\
1201 0.339263815453853\\
1351 0.37894417267084\\
1501 0.421433151231234\\
1651 0.467279090151664\\
1801 0.509784470239245\\
1951 0.550698720361289\\
2101 0.589097546914585\\
2251 0.624914280745062\\
2401 0.659759497427793\\
2551 0.690963386463071\\
2701 0.717935406868984\\
2851 0.741367658334342\\
3001 0.763143584481626\\
3151 0.781936414302958\\
3301 0.799036719848341\\
3451 0.814976073722672\\
3601 0.828237849637138\\
3751 0.839138435240317\\
3901 0.850328374939038\\
4051 0.859853947444202\\
4201 0.867422441882868\\
4351 0.873186343423007\\
4501 0.877175086239974\\
4651 0.881275674328729\\
4801 0.887345387908807\\
4951 0.890446122285143\\
5101 0.892443381845852\\
5251 0.890981135333009\\
5401 0.895118475679882\\
5551 0.896695103233126\\
5701 0.899839183797501\\
};
\addlegendentry{Proposed};

\addplot [
color=black,
solid,
mark=square,
mark options={solid}
]
table[row sep=crcr]{
1 0.111151050712311\\
151 0.127150740043951\\
301 0.143422997290899\\
451 0.163171051881058\\
601 0.1826184847844\\
751 0.204596362624235\\
901 0.228444141967848\\
1051 0.256460047159602\\
1201 0.286933161351528\\
1351 0.320836558291256\\
1501 0.355861460797147\\
1651 0.39331199774696\\
1801 0.430193765115312\\
1951 0.466496903047999\\
2101 0.504675233765267\\
2251 0.539387830583106\\
2401 0.5780671097534\\
2551 0.610694229591995\\
2701 0.641051452963876\\
2851 0.669307492312202\\
3001 0.697931684269331\\
3151 0.722355820080857\\
3301 0.750871994711866\\
3451 0.775657132512436\\
3601 0.784450342018595\\
3751 0.790095366409968\\
3901 0.768440277141001\\
4051 0.784035229702247\\
4201 0.75953888922944\\
4351 0.760639321046932\\
4501 0.79180068608314\\
4651 0.804909579499304\\
4801 0.735762869218844\\
4951 0.683846992276683\\
5101 0.731891312297916\\
5251 0.783180428345404\\
5401 0.803524601520081\\
5551 0.835097237376031\\
5701 0.863302314275531\\
5851 0.876043648682361\\
};
\addlegendentry{Old};

\end{axis}
\end{tikzpicture}%
\caption{The POP curves in new proposed method and old DLA-based for SSPP in Graph2(learning rate= 0.003)}
\label{figures.comparesspp}
\end{figure}
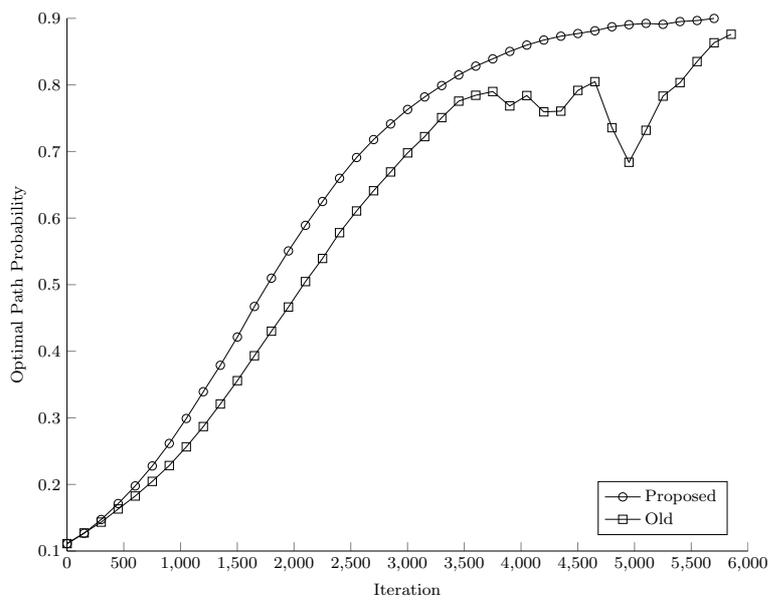
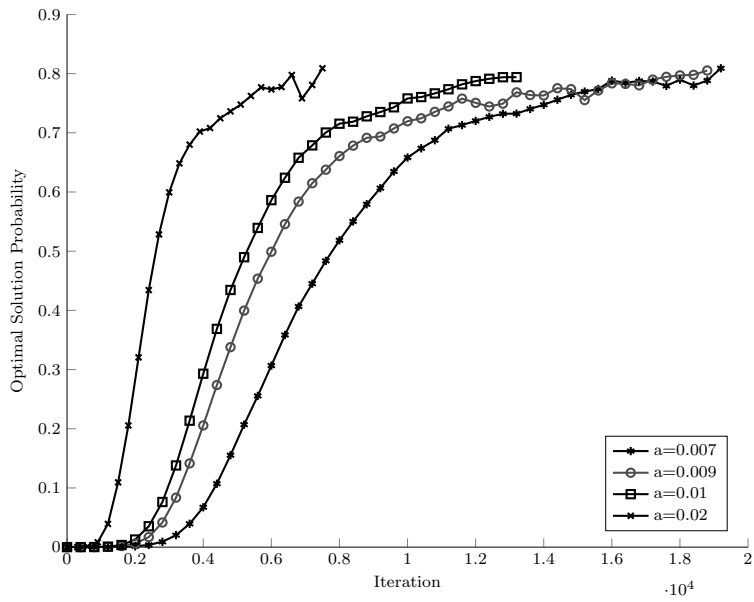
\begin{figure}
%
%
%
%
\begin{tikzpicture}

\begin{axis}[%
width=4.52083333333333in,
height=3.565625in,
scale only axis,
xmin=0,
xmax=20000,
xlabel={Iteration},
ymin=0,
ymax=0.9,
ylabel={Optimal Solution Probability},
axis x line*=bottom,
axis y line*=left,
legend style={at={(0.97,0.03)},anchor=south east,draw=black,fill=white,legend cell align=left}
]
\addplot [
color=black,
solid,
line width=1.0pt,
mark=asterisk,
mark options={solid}
]
table[row sep=crcr]{
1 3.34897976680384e-07\\
401 4.12851087261891e-06\\
801 2.86754460037536e-05\\
1201 0.000141900630969914\\
1601 0.000493131131691596\\
2001 0.0013954388479979\\
2401 0.00378404241020865\\
2801 0.00900902505451733\\
3201 0.0202287464668238\\
3601 0.0394130349974844\\
4001 0.0673083518198828\\
4401 0.107148884871234\\
4801 0.15534555073766\\
5201 0.206928747277437\\
5601 0.255471652897607\\
6001 0.306716253689218\\
6401 0.358453414207234\\
6801 0.406523653718021\\
7201 0.445100050054432\\
7601 0.483681579104976\\
8001 0.518444460543239\\
8401 0.550283094140485\\
8801 0.579305022673501\\
9201 0.606525227170528\\
9601 0.63467003181701\\
10001 0.658247767327248\\
10401 0.673867059407442\\
10801 0.687462228931241\\
11201 0.706835376312058\\
11601 0.7129973194098\\
12001 0.720109134696632\\
12401 0.726955303659614\\
12801 0.732144095197589\\
13201 0.732361928238244\\
13601 0.740327058681442\\
14001 0.747363993437415\\
14401 0.755833311208308\\
14801 0.764115526975794\\
15201 0.769590230912693\\
15601 0.77403972905921\\
16001 0.788197824050526\\
16401 0.784557483316199\\
16801 0.787609944937749\\
17201 0.786609150900472\\
17601 0.779324656357832\\
18001 0.789770750587207\\
18401 0.779807722146058\\
18801 0.788343170356646\\
19201 0.809012098465358\\
};
\addlegendentry{a=0.007};

\addplot [
color=gray!60!black,
solid,
line width=1.0pt,
mark=o,
mark options={solid}
]
table[row sep=crcr]{
1 3.34897976680384e-07\\
401 9.21735062334292e-06\\
801 8.91610916900099e-05\\
1201 0.000498162168183083\\
1601 0.00184241254777645\\
2001 0.00609263739799604\\
2401 0.0169553722683507\\
2801 0.041634602557218\\
3201 0.083620341937171\\
3601 0.141483323318638\\
4001 0.205586295064199\\
4401 0.273884743887512\\
4801 0.337870800580237\\
5201 0.399622916858577\\
5601 0.453618129058612\\
6001 0.499054882396754\\
6401 0.545719383973214\\
6801 0.583803250840114\\
7201 0.614785783847718\\
7601 0.637664796109857\\
8001 0.660713468686709\\
8401 0.678319779361356\\
8801 0.691495053799353\\
9201 0.693668105994568\\
9601 0.707352518687279\\
10001 0.719509909827557\\
10401 0.724486073253296\\
10801 0.735376610669943\\
11201 0.744527134282203\\
11601 0.75777982819966\\
12001 0.750618666695347\\
12401 0.744410920931109\\
12801 0.749284565148059\\
13201 0.768404315215925\\
13601 0.763719687330132\\
14001 0.763020781693661\\
14401 0.775355009019101\\
14801 0.773798710186875\\
15201 0.75523875040698\\
15601 0.77147064039139\\
16001 0.78351352862806\\
16401 0.782556603971792\\
16801 0.780138516540053\\
17201 0.789906518933804\\
17601 0.794545230189525\\
18001 0.797051935512777\\
18401 0.797917316899293\\
18801 0.805263718663956\\
};
\addlegendentry{a=0.009};

\addplot [
color=black,
solid,
line width=1.0pt,
mark=square,
mark options={solid}
]
table[row sep=crcr]{
1 3.34897976680384e-07\\
401 1.24180421360532e-05\\
801 0.000139139228844554\\
1201 0.000791749364092676\\
1601 0.00354561635217511\\
2001 0.0128321916665138\\
2401 0.0352836693200567\\
2801 0.0762600726760442\\
3201 0.138010353875981\\
3601 0.213635036789568\\
4001 0.292967898860457\\
4401 0.368708475231076\\
4801 0.434461402397646\\
5201 0.489826400792369\\
5601 0.539275551211236\\
6001 0.586223965955238\\
6401 0.624053683744141\\
6801 0.657827270117652\\
7201 0.678751552131665\\
7601 0.700483233186986\\
8001 0.715343161917921\\
8401 0.718744213275032\\
8801 0.727811223760416\\
9201 0.735066396617926\\
9601 0.743076770757846\\
10001 0.758035699717323\\
10401 0.760362704285421\\
10801 0.766568887586689\\
11201 0.77355430013305\\
11601 0.781874742066999\\
12001 0.787370653491496\\
12401 0.791253475459276\\
12801 0.793903773883456\\
13201 0.794020343489613\\
};
\addlegendentry{a=0.01};

\addplot [
color=black,
solid,
line width=1.0pt,
mark=x,
mark options={solid}
]
table[row sep=crcr]{
1 3.34897976680384e-07\\
301 5.73338123300809e-05\\
601 0.001074571567878\\
901 0.008203694539561\\
1201 0.0392122111194932\\
1501 0.109434876241114\\
1801 0.205552174001043\\
2101 0.320575260462467\\
2401 0.43432382502284\\
2701 0.528262921358092\\
3001 0.599398166065634\\
3301 0.648270138150579\\
3601 0.680056637048028\\
3901 0.702127409329046\\
4201 0.708014794237662\\
4501 0.724596312819104\\
4801 0.736227114773899\\
5101 0.747730587150078\\
5401 0.762233613914635\\
5701 0.776856700957435\\
6001 0.773261498137945\\
6301 0.777296567122222\\
6601 0.797600785613711\\
6901 0.758204174609647\\
7201 0.780772039066647\\
7501 0.808980614253245\\
};
\addlegendentry{a=0.02};

\end{axis}
\end{tikzpicture}%
\caption{Probability of optimal spanning tree in eDLA-based algorithm for different learning rates in Alex1-a}
\label{figures.khaliliallpopsmstp}
\end{figure}
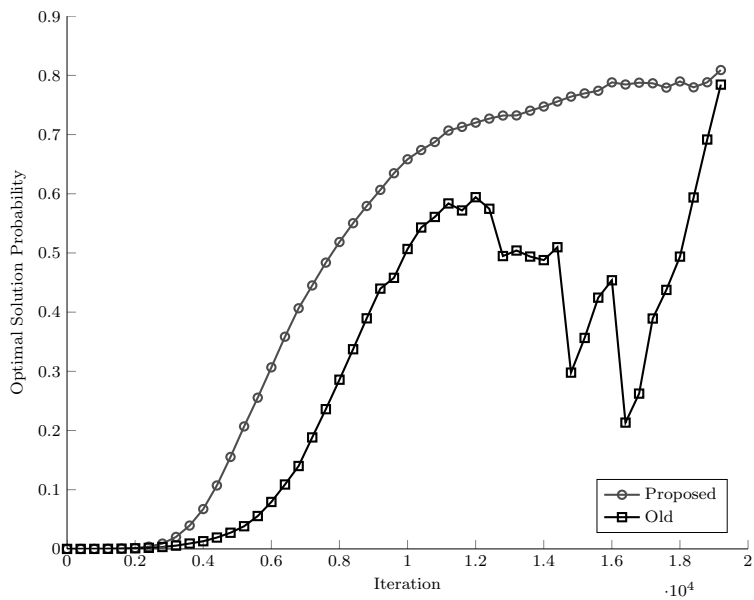
\begin{figure}
%
%
%
%
\begin{tikzpicture}

\begin{axis}[%
width=4.52083333333333in,
height=3.565625in,
scale only axis,
xmin=0,
xmax=20000,
xlabel={Iteration},
ymin=0,
ymax=0.9,
ylabel={Optimal Solution Probability},
axis x line*=bottom,
axis y line*=left,
legend style={at={(0.97,0.03)},anchor=south east,draw=black,fill=white,legend cell align=left}
]
\addplot [
color=gray!60!black,
solid,
line width=1.0pt,
mark=o,
mark options={solid}
]
table[row sep=crcr]{
1 3.34897976680384e-07\\
401 4.12851087261891e-06\\
801 2.86754460037536e-05\\
1201 0.000141900630969914\\
1601 0.000493131131691596\\
2001 0.0013954388479979\\
2401 0.00378404241020865\\
2801 0.00900902505451733\\
3201 0.0202287464668238\\
3601 0.0394130349974844\\
4001 0.0673083518198828\\
4401 0.107148884871234\\
4801 0.15534555073766\\
5201 0.206928747277437\\
5601 0.255471652897607\\
6001 0.306716253689218\\
6401 0.358453414207234\\
6801 0.406523653718021\\
7201 0.445100050054432\\
7601 0.483681579104976\\
8001 0.518444460543239\\
8401 0.550283094140485\\
8801 0.579305022673501\\
9201 0.606525227170528\\
9601 0.63467003181701\\
10001 0.658247767327248\\
10401 0.673867059407442\\
10801 0.687462228931241\\
11201 0.706835376312058\\
11601 0.7129973194098\\
12001 0.720109134696632\\
12401 0.726955303659614\\
12801 0.732144095197589\\
13201 0.732361928238244\\
13601 0.740327058681442\\
14001 0.747363993437415\\
14401 0.755833311208308\\
14801 0.764115526975794\\
15201 0.769590230912693\\
15601 0.77403972905921\\
16001 0.788197824050526\\
16401 0.784557483316199\\
16801 0.787609944937749\\
17201 0.786609150900472\\
17601 0.779324656357832\\
18001 0.789770750587207\\
18401 0.779807722146058\\
18801 0.788343170356646\\
19201 0.809012098465358\\
};
\addlegendentry{Proposed};

\addplot [
color=black,
solid,
line width=1.0pt,
mark=square,
mark options={solid}
]
table[row sep=crcr]{
1 3.34711119369494e-07\\
401 5.33891804105005e-06\\
801 3.74130181700562e-05\\
1201 0.000161844494929389\\
1601 0.000458845173585486\\
2001 0.00103960030727002\\
2401 0.0018079166300697\\
2801 0.00323033869701837\\
3201 0.00555169972676764\\
3601 0.00908951163761367\\
4001 0.0129899075944441\\
4401 0.0192042282706784\\
4801 0.0273164266506076\\
5201 0.0383220546982353\\
5601 0.0555510322464302\\
6001 0.079306416350305\\
6401 0.108807330482693\\
6801 0.140036146945418\\
7201 0.188355465578228\\
7601 0.236090048728508\\
8001 0.285873865020949\\
8401 0.337322349447056\\
8801 0.389544354834942\\
9201 0.439714396329893\\
9601 0.457962994174392\\
10001 0.506685314811637\\
10401 0.542962953594824\\
10801 0.560822180880179\\
11201 0.583639539842688\\
11601 0.571776389664792\\
12001 0.594190064744931\\
12401 0.574641656982527\\
12801 0.494748746848123\\
13201 0.504103246608843\\
13601 0.493903937872525\\
14001 0.487918820074886\\
14401 0.509692219027978\\
14801 0.297819182494434\\
15201 0.356431467754914\\
15601 0.424453260025615\\
16001 0.45384553335644\\
16401 0.21342321968782\\
16801 0.262372560527152\\
17201 0.389232436445517\\
17601 0.437558306214332\\
18001 0.493796996576047\\
18401 0.593749660419928\\
18801 0.691617537052403\\
19201 0.784382370058318\\
};
\addlegendentry{Old};

\end{axis}
\end{tikzpicture}%
\caption{The curves of probability of optimal spanning tree  in new proposed method and old DLA-based and LA based for SMSTP(learning rate= 0.07) in Alex1-a}\label{figures.akbarikhalilicmp}
\end{figure}
\endgroup
\section{Conclusion:}
In this paper eDLA as a new method for cooperating of learning automaton is introduced.  In this network each learning automata has an activity level that changes over the time according to the underlying problem and a set of communication rules. At each step of algorithm only one learning automata has high level so-called as Fire.\\
A new adaptive procedure based on the eDLA is introduced to solving optimization problems in stochastic graphs. The algorithm presented  provides policy that can be used to determine set of edges to be sampled. As a result of that the algorithm determines a sub-graph with optimal expected weight. \\
The convergence and optimality of the proposed algorithm is proved. Two set of experiments are done to solving stochastic shortest path and stochastic minimum spanning tree problems. The results indicate the superiority of the proposed algorithm compared to previous automata based algorithms. Further more, it seems that the new proposed threshold value is useful to prevent network optimization from becoming stuck in local optima and improves the convergence of algorithm significantly.

  A new variance aware method for computation of threshold value is introduced that improves the accuracy and convergence rate of the proposed algorithm.\\
The convergence of the proposed algorithm shown.
    

\bibliographystyle{spmpsci}      
\bibliography{mycites}   

\begin{thebibliography}{10}
\providecommand{\url}[1]{{#1}}
\providecommand{\urlprefix}{URL }
\expandafter\ifx\csname urlstyle\endcsname\relax
  \providecommand{\doi}[1]{DOI~\discretionary{}{}{}#1}\else
  \providecommand{\doi}{DOI~\discretionary{}{}{}\begingroup
  \urlstyle{rm}\Url}\fi

\bibitem{A.Alipour2005}
A.Alipour, M.R.Meybodi: {Solving Traveling Salesman Problem Using Distributed
  Learning Automata}.
\newblock In: 10th Annual CSI Computer Conference, pp. 759--761. Tehran,Iran
  (2005)

\bibitem{AkbariTorkestani2011}
{Akbari Torkestani}, J., Meybodi, M.R.: {Learning automata-based algorithms for
  solving stochastic minimum spanning tree problem}.
\newblock Applied Soft Computing \textbf{11}(6), 4064--4077 (2011).
\newblock \doi{10.1016/j.asoc.2011.02.017}.
\newblock
  \urlprefix\url{http://linkinghub.elsevier.com/retrieve/pii/S1568494611000779}

\bibitem{Anari2007}
Anari, B., Meybodi, M.R.: {A Method based on distributed learning automata for
  determining web documents structure}.
\newblock In: 12th Annual CSI Computer Conference of Iran, pp. 2276--2282
  (2007)

\bibitem{BaradaranHashemi2007}
BaradaranHashemi, A., Meybodi, M.: {Web Usage Mining Using Distributed Learning
  Automata}.
\newblock In: 12th Annual CSI Computer Conference of Iran, pp. 553--560.
  Tehran,Iran (2007)

\bibitem{Beigy2003}
Beigy, H., Meybodi, M.: {Solving Stochastic Shoretst Path Problem Using Monte
  Carlo Sampling Method: A Distributed Learning Automata Approach}.
\newblock Springer-Verlag Lecture Notes in Advances in Soft Computing: Neural
  Networks and Soft Computing pp. 626--632 (2003)

\bibitem{H.Beigy2006}
H.Beigy, M.R.Meybodi: {Utilizing distributed learning automata to solve
  stochastic shortest path problems}.
\newblock International Journal of Uncertainty, \ldots \textbf{14}(5), 591--615
  (2006).
\newblock
  \urlprefix\url{http://www.worldscientific.com/doi/abs/10.1142/S0218488506004217}

\bibitem{Hutson}
Hutson, K.R., Shier, D.R.: {Online Supplement to “ Bounding Distributions for
  the Weight of a Minimum Spanning Tree in Stochastic Networks ”} pp. 1--12.
\newblock
  \urlprefix\url{http://www.math.clemson.edu/~shierd/Shier/appendix1.pdf}

\bibitem{Lakshmivarahan1976}
Lakshmivarahan, S., Thathachar, M.: {Bounds on the Convergence Probabilities of
  Learning Automata}.
\newblock IEEE Transactions on Systems, Man, and Cybernetics - Part A: Systems
  and Humans \textbf{6}(11), 756--763 (1976)

\bibitem{Meybodi2002}
Meybodi, M.R., Beigy, H.: {A Sampling Method Based on Distributed Learning
  Automata}.
\newblock In: the 10th iranian conference on Electrical Engineering, vol.~I
  (2002)

\bibitem{MollakhaliliMeybodi2008}
MollakhaliliMeybodi, M., Meybodi, M.: {Link Prediction in Adaptive Web Sites
  Using Distributed Learning Automata,}.
\newblock In: 13th Annual CSI Computer Conference of Iran. Kish Island (2008)

\bibitem{MollakhaliliMeybodi2012}
MollakhaliliMeybodi, M., Meybodi, M.R.: {A Distributed Learning Automata Based
  Approach for User Modeling in Adaptive Hypermedia}.
\newblock In: Congress on Electrical, Computer and Information Technology.
  Mashhad (2012)

\bibitem{MollakhaliliMeybodi2004}
MollakhaliliMeybodi, M., M.R.Meybodi: {A New Distributed Learning Automata
  Based Algorithm for Solving Stochastic Shortest Path}.
\newblock In: 6th Conference on Intelligent Systems. kerman (2004)

\bibitem{Motevalian2006}
Motevalian, A., Meybodi, M.: {Solving Maximal Independent Set Problem Using
  Distributed Learning Automata}.
\newblock In: 14th Iranian Electrical Engineering Conference(ICEE2006), vol.~1.
  Tehran,Iran (2006)

\bibitem{Narendra1974}
Narendra, K.S., Thathachar, M.A.L.: {Learning Automata: A Survey}.
\newblock IEE Tansactions on systems, man, and cyberentics \textbf{SMC-14}(4),
  323--334 (1974)

\bibitem{Norman1968}
Norman, F.: {On the Linear Model with Two Absorbing}.
\newblock Journal of Mathematical Psychology \textbf{5}, 225--241 (1968)

\bibitem{Papoulis1991}
Papoulis, A.: {Probability, Random Variables, and Stochastic Processes}, 3 edn.
\newblock McGrawHill, New York, USA (1991)

\bibitem{Ross2004}
Ross, S.M.: {Introduction to Probability and Statistics dor Engineers and
  Scientists}, third edn.
\newblock Elsevier Academic Press (2004)

\bibitem{S.1989}
S., N.K., L, T.M.: {Learning Automata: an Introduction}.
\newblock Prentice Hall (1989)

\bibitem{Saati2005}
Saati, S., Meybodi, M.: {A Self Organizing Model for Document Structure Using
  Distributed Learning Automata}.
\newblock In: Second International Conference on Information and Knowledge
  Technology (IKT2005) (2005)

\bibitem{Sato1999}
Sato, T.: {On Some Asymptotic Properties of Learning Automaton Networks}.
\newblock Tech. rep. (1999).
\newblock
  \urlprefix\url{http://citeseerx.ist.psu.edu/viewdoc/summary?doi=10.1.1.41.7669}

\bibitem{Thathachar2002}
Thathachar, M.L., Sastry, P.S.: {Varieties of learning automata: an overview.}
\newblock IEEE transactions on systems, man, and cybernetics. Part B,
  Cybernetics : a publication of the IEEE Systems, Man, and Cybernetics Society
  \textbf{32}(6), 711--22 (2002).
\newblock \doi{10.1109/TSMCB.2002.1049606}.
\newblock \urlprefix\url{http://www.ncbi.nlm.nih.gov/pubmed/18244878}

\bibitem{ThathacharMALandHarita1987}
{Thathachar , MAL and Harita}, B.: {Learning automata with changing number of
  actions}.
\newblock IEEE Transactions on Systems, Man, and Cybernetics - Part A: Systems
  and Humans \textbf{17}(6), 1095--1100 (1987).
\newblock \urlprefix\url{http://portal.acm.org/citation.cfm?id=40875.40898}

\end{thebibliography}
\end{document}